\documentclass[letterpaper, 10 pt, onecolumn]{ieeeconf}

\usepackage{lastpage} 
\usepackage{extramarks} 
\usepackage{graphicx} 
\usepackage{booktabs}
\usepackage{amsmath}
\usepackage{amssymb}

\usepackage{amsthm}

\usepackage{algorithm}
\usepackage{algpseudocode}
\usepackage{mathtools}
\mathtoolsset{showonlyrefs}
\usepackage{multirow}

\newcommand{\norm}[1]{\left\lVert#1\right\rVert}

\newcommand{\policy}{\pi}
\newcommand{\jointState}{s}
\newcommand{\jointStateSpace}{\mathcal{S}}
\newcommand{\jointAction}{a}
\newcommand{\jointActionSpace}{\mathcal{A}}
\newcommand{\idist}{\alpha}

\newcommand{\observedTraj}{\tau}
\newcommand{\unknownPara}{\theta}
\newcommand{\featureVec}{\phi}

\newcommand{\expectedStateVisit}{d}
\newcommand{\estReward}{\Tilde{r}}

\newcommand{\totalReward}{R}
\newcommand{\estTotalReward}{\Tilde{\totalReward}}
\newtheorem{oproblem}{Optimization Problem}
\newtheorem{definition}{Definition}
\newtheorem{theorem}{Theorem}
\newtheorem{proposition}{Proposition}
\newtheorem{remark}{Remark}
\newtheorem{lemma}{Lemma}
\newtheorem{problem}{Problem}

\usepackage{tikz}
\usepackage{pgfplots}
\usepackage{hyperref}
\usepgfplotslibrary{groupplots,dateplot}
\usetikzlibrary{patterns,shapes.arrows, fadings}
\usetikzlibrary{automata} 
\usetikzlibrary{shapes.geometric}
\usetikzlibrary{calc,shadows}
\usetikzlibrary{positioning} 
\usetikzlibrary{arrows} 
\tikzset{node distance=2.5cm, 
every state/.style={ 
semithick,
fill=gray!10},
initial text={}, 
double distance=2pt, 
every edge/.style={ 
draw,
->,>=stealth', 
auto,
semithick}}
\pgfplotsset{compat=newest}
\newlength\figH
\newlength\figW
\title{Deceptive Sequential Decision-Making via Regularized Policy Optimization}
\author{Yerin Kim$^{1}$, Alexander Benvenuti$^{1}$, Bo Chen$^{1}$, Mustafa Karabag$^{2}$, Abhishek Kulkarni$^{2}$, Nathaniel D. Bastian$^{3}$, Ufuk Topcu$^{2}$, Matthew Hale$^{1}$
\thanks{
$^{1}$Georgia Institute of Technology, Atlanta, GA, USA.
Emails: \texttt{\{yerinkim, abenvenuti3, bchen351, matthale\}@gatech.edu}.
}
\thanks{
$^{2}$University of Texas at Austin, Austin, TX, USA.
Emails: \texttt{\{karabag, abhishek.kulkarni, utopcu\}@utexas.edu}.
}
\thanks{
$^{2}$United States Military Academy, West Point, NY, USA.
Emails: \texttt{\{nathaniel.bastian\}@westpoint.edu}.
}
}
\date{August 2025}

\begin{document}

\maketitle

\begin{abstract}
Autonomous systems are increasingly expected to operate in the presence of adversaries, though adversaries may infer sensitive information simply by observing a system. Therefore, present a deceptive sequential decision-making framework that not only conceals sensitive information, but actively misleads adversaries about it. We model autonomous systems as Markov decision processes, with adversaries using inverse reinforcement learning to recover reward functions. To counter them, we present three regularization strategies for policy synthesis problems that actively deceive an adversary about a system's reward.
``Diversionary deception'' leads an adversary to draw any false conclusion about the system's reward function. ``Targeted deception'' leads an adversary to draw a specific false conclusion about the system's reward function. ``Equivocal deception'' leads an adversary to infer that 
the real reward and a false reward both explain the system's behavior. 
We show how each form of deception can be implemented in policy optimization problems and analytically bound the loss in total accumulated reward induced by deception. 
Next, we evaluate these developments in a multi-agent setting. We show that diversionary, targeted, and equivocal deception all steer the adversary to false beliefs while still attaining a total accumulated reward 
that is at least~$98\%$ of its optimal, non-deceptive value. 
\end{abstract}

\begin{keywords}
    Deception, Markov Decision Processes, Inverse Reinforcement Learning
\end{keywords}

\section{Introduction}\label{sec:Intro}

Autonomous systems are used in various forms of vital infrastructure, including manufacturing systems~\cite{monostori2016cyber,arkin1990autonomous}, autonomous vehicles~\cite{chen2017cyber,guo2022cyber}, and smart power grids~\cite{yu2016smart}. 
One challenge in using these systems is that they can be observed by external parties
that may seek to uncover and exploit sensitive information~\cite{qayyum2020securing,SHEEHAN2019523}. 
For example, an adversary may observe the daily routines of an individual's
autonomous vehicle, which can give insight into when that individual
is home or away, who their close associates are, and other sensitive information. 

In military or cybersecurity contexts where strict confidentiality is required, such observations can result in severe operational loss~\cite{10.1145/989.991}.
For example, if an adversary infers the underlying objective in troop movements or resource allocation, then they might exploit this knowledge to make targeted attacks~\cite{mhara2024cyber}.
Moreover, recurring patterns on supply channels or communication frequencies can unintentionally disclose strategic priorities~\cite{9152764}.
Unfortunately, such observations are often unavoidable,
in the sense that a user cannot stop someone else from physically observing them.
Thus, another approach is needed to 
reduce the leakage of sensitive information.





One way of mitigating risk is by adopting privacy to protect sensitive data. Traditional privacy 
implementations often induce uncertainty, thereby giving adversaries ambiguous information~\cite{benvenuti2023differentially,9304015}.
However, adversaries may still glean sensitive information despite the implementation of privacy protections~\cite{9833672,9519418}, and simply inducing uncertainty may not be sufficient 
to deter an adversary or mitigate the effects of their efforts. 
For example, an adversary may still predict specific times that a user is not at home, even
if the user's exact departure and arrival times are uncertain. 
As a result, interest has arisen in the development of deceptive techniques that drive observers to draw incorrect conclusions about the systems that they observe~\cite{karabag2021deception,mceneaney2005deception,lv2024optimal}.

In this paper, we develop and analyze three forms of deception for sequential decision systems.
We model such systems as Markov decision processes (MDPs), and we seek to deceive an adversary about an MDP's reward function, which encodes its objectives and intentions.
The adversary is modeled as using inverse reinforcement learning (IRL) to attempt to infer
objectives, and we introduce three types of deception to counter these efforts. 
The first is ``diversionary'' deception, which
seeks to make an adversary draw any incorrect conclusion about a system's reward.
The second is ``targeted'' deception, 
which seeks to make an adversary draw a particular incorrect
conclusion about a system's reward.
The third is ``equivocal'' deception, which seeks to make an adversary draw a false conclusion about the reward function that introduces ambiguity between two different objectives.
For each, we formulate tractable optimization problems for
policy synthesis that include regularizer terms that implement 
each type of deception. 

In detail, our contributions are: 
\begin{itemize}
    \item We introduce ``diversionary deception'', ``targeted deception'' and  ``equivocal deception'' and formulate a family of optimization
    problems for synthesizing decision policies that implement each one
    (Definitions~\ref{def:diversionary_deception}, \ref{def:targeted_deception}, and~\ref{def:equivocal_deception}
    and Optimization Problems~\ref{op:diversionary_optimization_problem}, \ref{op:targeted_optimization_problem}, and~\ref{op:modified_equivocal_optimization_problem}).
    
    \item We analytically bound the loss in total accumulated reward that is caused by behaving deceptively 
    in terms of a user-specified parameter that is used to implement
    each type of deception
    (Theorems~\ref{thm:bound_div}, ~\ref{thm:bound_tar} and~\ref{thm:bound_equ}).
    \item We validate the deceptiveness and performance of our method in numerical simulations of a network defense problem, and we empirically
    assess deception/performance tradeoffs
    (Section~\ref{sec:result}).
\end{itemize}

Existing work in~\cite{kim2024defining} has presented related definitions of deception and
empirical results for them, but to the best of our knowledge the current paper is the first
to analytically characterize these types of deception and the tradeoffs they create. 
In addition, this work expands the empirical results by exploring deception and performance under different algorithms of inverse reinforcement learning, ensuring their reliability in applications.

\subsection{Related Works}



Privacy has been widely studied for protecting sensitive information of 
decision systems at runtime.
For example, the work in~\cite{yazdani2022differentially,hawkins2020differentially} adopted differential privacy to protect agents' state trajectories in multi-agent systems.
By using differential privacy, 
the authors in~\cite{benvenuti2023differentially} privatize reward functions in policy synthesis for multi-agent MDPs, while those in~\cite{chen2023differentialsymbolic,chen2023differential} ensure protection of non-numeric or symbolic data.
An entropy-based approach has been considered as another method of protecting information. In~\cite{karabag2019least} and~\cite{savas2019entropy}, the information leaked to the observer is minimized to reduce the observer's ability to infer 
certain characteristics of agents, such as their transition probabilities. 

As described above, while privacy can induce uncertainty in an observer, 
a need has arisen for deceptive decision strategies that deliberately
induce false beliefs in an observer.  For example,
the authors in~\cite{abdulhai2024defining} defined deception in 
a partially observable
speaker-listener problem, and deception is formalized as the regret
that the listener incurs by listening to the speaker. 
Deceptive strategies
of agents are explored in~\cite{karabag2021deception} and~\cite{karabag2022exploiting} in environments in which a supervisor provides a reference policy for an agent to follow,
and the agent attempts to follow a different policy while giving
the appearance of following the policy that was specified by the supervisor.

The current paper differs from privacy-based approaches because our focus is 
on deliberately inducing incorrect inferences, rather than generating uncertainty.
In addition, this work is different from~\cite{karabag2021deception} and~\cite{karabag2022exploiting} because we consider a system operating under observation of adversaries rather than a supervisor. Thus, no reference policy is given
to the agents being observed in this work, and the observer has no baseline
expectations for the agents' behavior. We differ from~\cite{abdulhai2024defining}
because we define deception in terms of an observer's beliefs, rather than
their reward (and our definition does not even require the observer to have
a reward). 

The rest of this paper is organized as follows. Section~\ref{sec:prelims} presents preliminaries on MDPs and IRL followed by formal problem statements.  
Then Section~\ref{sec:method} 
defines optimization problems that implement deception and 
derives performance guarantees for 
the resulting deceptive decision
policies. 
Then
Section~\ref{sec:result} validates these finding with numerical simulations, 
and Section \ref{sec:conclude} concludes.

\textbf{Notation} We use~$\mathbb{R}$ to denote the real numbers, and
we use~$\Delta(S)$ to denote the set of probability distributions
over a finite set~$S$. 


\section{Preliminaries and Problem Statements} \label{sec:prelims}

In this section, we introduce preliminaries for Markov decision processes and inverse reinforcement learning, 
then define deception, and lastly give formal problem statements.

\subsection{Preliminaries on Markov Decision Processes}\label{subsec:MDP}

We consider a collection of $M$ agents indexed over $i\in\{1,\dots,M\}$ and 
modeled as Markov decision processes.

\begin{definition}[Markov Decision Process]
A Markov decision process (MDP) $\mathcal{M}^i$ is a tuple
$\mathcal{M}^i = (\mathcal{S}^i, \mathcal{A}^i, r^i, \mathcal{T}^i, \idist ^i)$, where
$\mathcal{S}^i$ is a finite set of states,
$\mathcal{A}^i$ is a finite set of actions,
$r^i:\mathcal{S}^i\times\mathcal{A}^i\rightarrow\mathbb{R}$ is a reward function,
$\mathcal{T}^i:\mathcal{S}^i\times \mathcal{A}^i \rightarrow \Delta(\mathcal{S}^i)$ is a transition probability function, and
$\idist^i \in \Delta(\mathcal{S}^i)$ 
is the probability distribution of the initial state. 
\end{definition}

For notational simplicity, we define $\mathcal{T}^i(s^i,a^i, y^i)$  as the probability of transitioning to state $y^i \in \mathcal{S}^i$ when action $a^i\in \mathcal{A}^i$ is taken in state $s^i\in \mathcal{S}^i$. 
Also,~$\alpha^i(s^i)$ denotes the probability of the initial state being $s^i \in \mathcal{S}^i$. 
To model systems composed of multiple MDPs, we next define multi-agent MDPs (MMDPs).

\begin{definition}[Multi-Agent Markov Decision Process; \cite{boutilier1996planning}]
A Multi-Agent Markov Decision Process (MMDP) $\mathcal{M}$ is a tuple
$\mathcal{M} = (\mathcal{S}, \mathcal{A}, r, \gamma, \mathcal{T}, \alpha)$, where
$\mathcal{S} = \mathcal{S}^1 \times \cdots \times  \mathcal{S}^M$ is a set of joint states, 
$\mathcal{A} = \mathcal{A}^1 \times \cdots \times  \mathcal{A}^M$ is a set of joint actions, 
$r:\mathcal{S}\times\mathcal{A}\rightarrow\mathbb{R}$ is a joint reward function,
$\gamma \in [0,1)$ is a discount factor,
$\mathcal{T}:\mathcal{S}\times \mathcal{A} \rightarrow \Delta(\mathcal{S})$ is a transition probability function over the joint state space and joint action space, 
and 
$\idist \in  \Delta(\mathcal{S})$ 
is a distribution over 
initial joint states.
We use $\mathcal{T}(s,a,y) = \prod^M_{i=1}\mathcal{T}^i(s^i, a^i, y^i)$ to denote the transition probability from state $s$ to state $y$ by taking action $a$, where  $s = (s^1, \dots, s^M)\in \mathcal{S}$,  $a = (a^1, \dots, a^M)\in \mathcal{A}$, and $y = (y^1, \dots, y^M)\in \mathcal{S}$.
We define $\idist(s) = \prod^M_{i=1}\idist^i(s^i)$ as the probability of 
the initial joint state being $s \in \mathcal{S}$.
\end{definition}

A joint policy $\policy:\mathcal{S} \rightarrow \Delta(\mathcal{A})$ is defined as $\policy = (\policy^1, \dots, \policy^M)$, where, for each~$i$, agent $i$ follows the 
policy $\policy^i: \mathcal{S}^i \rightarrow\Delta(\mathcal{A}^i)$. 
We abuse notation and say that~$\policy(a \mid s)$  is the probability of taking joint action $a\in \mathcal{A}$ in joint state $s\in\mathcal{S}$.

\begin{remark}
An MMDP is a particular type of MDP, and all forthcoming statements
about MDPs also apply to MMDPs. 
\end{remark}

The goal of an MMDP is to compute an optimal joint policy that maximizes the value function 
\begin{equation}
v_\policy(s)=\mathbb{E}\left[\sum_{n=1}^\infty \gamma^n \sum_{a_n \in \mathcal{A}}r(s_n,a_n)\policy(a_n \mid s_n)\right],
\end{equation} 
where $s_n\in \mathcal{S}$ and $a_n \in \mathcal{A}$ are the joint state and joint action at timestep~$n$, respectively. 
The value function at~$s$, denoted~$v_{\policy}(s)$, 
is equal to the expected total discounted reward that is
accumulated when starting from state~$s$ and using
decision policy~$\policy$.  
An optimal policy that maximizes this value function can be found 
efficiently via linear programming~\cite{puterman2014markov,ying2020note}.

\begin{lemma}  [MDP LP;~{\cite[Section 6.9.1]{puterman2014markov}}]\label{lem:puterman_probs}
The following optimization problem computes the optimal value function $v^*(s)$ for all $s \in \mathcal{S}$: 
\textnormal{
\begin{oproblem} \label{op:LP}
\begin{equation}
    \begin{aligned}
    &\begin{aligned}
        \underset{ v_{\policy} \in \mathbb{R}^{|\mathcal{S}|}}{\operatorname{minimize}} &\quad \sum_{s\in\mathcal{S}}\idist(s)v_{\policy}(s)
    \end{aligned}
        \\
        &\begin{aligned}
            \textnormal{ subject to }\,\, &v_{\policy}(s)- \gamma\sum_{j\in \mathcal{S}}\mathcal{T}(s,a,j)v_{\policy}(j)\geq r(s,a)  \textnormal{ for all } s \in \mathcal{S}, a \in \mathcal{A}.
        \end{aligned}
    \end{aligned}   
\end{equation}
\end{oproblem}
}
\end{lemma}

Optimization Problem~\ref{op:LP} may also be solved by solving
its dual:
\begin{oproblem}\label{op:dualLP}
\begin{equation}
    \begin{aligned}
    &\begin{aligned}
        & \underset{ x\in \mathbb{R}^{|\mathcal{S}||\mathcal{A}|}}{\operatorname{maximize}} \quad \sum_{s\in S}\sum_{a\in A}r(s, a)x(s, a)
    \end{aligned}
        \\
        &\begin{aligned}
            \textnormal{ subject to }  & \sum_{a \in \mathcal{A}} x(j,a) - \gamma \sum_{s \in \mathcal{S}} \sum_{a \in \mathcal{A}} \mathcal{T}(s,a,j) x(s,a) = \idist (j) \quad \textnormal{ for all } j \in \mathcal{S}\\
            &x(s, a) \geq 0\quad \textnormal{ for all } s \in \mathcal{S}, a \in \mathcal{A}.
        \end{aligned}
    \end{aligned}    
\end{equation} 
\end{oproblem}

We define $P^\policy_{n,j}(s,a)$ as the probability of an MMDP being in state $s$ and taking action $a$ at timestep $n$ when the initial state is $j$ and 
the MMDP's policy is $\policy$. 

\begin{definition} \label{def:ocmeas}
Let an MDP be given. 
For a policy $\pi$, the associated \emph{occupancy measure} is~$\{x_\policy(s,a)\}_{s\in\mathcal{S},a\in\mathcal{A}}$, 
defined as 
\begin{align}\label{eq:def_occupancy_measure}
x_\policy(s,a) = \sum_{j\in \mathcal{S}}\idist(j)\sum_{n=1}^\infty \gamma^{n-1} P^\policy_{n,j}(s,a)
\end{align} 
for all $s\in \mathcal{S}$ and $a \in \mathcal{A}$.
\end{definition}

An occupancy measure~$\{x_\policy(s,a)\}_{s\in\mathcal{S},a\in\mathcal{A}}$ can be interpreted as the frequency with which each state-action pair is occupied under the initial distribution $\idist$ and the policy~$\policy$. 
Moreover, given any policy~$\pi$, the corresponding occupancy measure~$x_\pi = \{x_\policy(s,a)\}_{s\in\mathcal{S},a\in\mathcal{A}}$ is in the feasible region for Optimization Problem~\ref{op:dualLP}~\cite{puterman2014markov}.
Throughout the remainder of the paper, with an abuse of notation,
any solution $x$ of Optimization Problem~\ref{op:dualLP} will refer to the corresponding set of occupancy measures~$\{x(s,a)\}_{s\in\mathcal{S},a\in\mathcal{A}}$.

Given an occupancy measure $x$, we can compute the policy for each state-action pair as
\begin{equation}\label{eq:oc_policy}
    \policy(a \mid s) = \frac{x(s,a)}{\sum_{a'\in \mathcal{A}} x(s,a')},
\end{equation}
which gives~$x_\policy(s,a) = x(s,a)$ for all $s\in \mathcal{S}$ and $a\in\mathcal{A}$, and this relation establishes a one-to-one relationship between policies and occupancy measures. 
In addition, if $x^*=\{x^*(s,a)\}_{s \in \mathcal{S}, a \in \mathcal{A}}$ is a solution to Optimization Problem \ref{op:dualLP}, then the policy obtained from~\eqref{eq:oc_policy} 
using~$x^*$ is an optimal policy~\cite{puterman2014markov}. This optimization-based approach of synthesizing policies has seen wide use in similar privacy and deception work~\cite{benvenuti2024guaranteed,chen2023differential,karabag2022exploiting} for the ability to add additional constraints to the policy synthesis, which we leverage in this work. Specifically, we desire frequent visitation to certain key states which we refer to as ``goal states". 

\begin{definition} \label{def:goal_states}
    A user-specified set $\mathcal{S}_{goal}$ is a set of goal states if it satisfies
    \[
     \max_{a\in \mathcal{A}}r(s_g,a) \geq  \max_{a\in \mathcal{A}}r(s,a)
    \]
    for all $s_g\in S_{goal}$ and $s\in S\setminus S_{goal}$. 
    Moreover, every $s_g \in \mathcal{S}_{goal}$ is a goal state.
\end{definition}

Since the system aims to visit goal states, 
we add an additional constraint to Optimization Problem~\ref{op:dualLP}, which enforces frequent visitation to them:
\begin{oproblem}\label{op:dualLP_with_task_constraint}
\begin{equation}
\begin{aligned}
\underset{ x\in \mathbb{R}^{|\mathcal{S}||\mathcal{A}|}}{\operatorname{maximize}} \,\,& \sum_{s \in \mathcal{S}} \sum_{a \in \mathcal{A}}   r(s,a) x(s,a) \\
\textnormal{ subject to } 
&\sum_{a \in \mathcal{A}} x(j,a) \!-\! \gamma \sum_{s \in \mathcal{S}} \sum_{a \in \mathcal{A}} \mathcal{T}(s,a,j) x(s,a)  = \idist (j) \quad \textnormal{ for all } j \in \mathcal{S}\\
&\sum_{q \in \mathcal{S}_{goal}}  \sum_{s\in \mathcal{S}} \sum_{a \in \mathcal{A}}  \mathcal{T}(s,a,q) x(s,a)  \geq v_{reach} \\
    & x(s,a) \geq 0 \quad \textnormal{ for all } s\in \mathcal{S},a\in \mathcal{A},  
\end{aligned}
\end{equation}
\end{oproblem}
where the threshold $v_{reach}$ determines the minimum
required amount of expected visitation to states in $\mathcal{S}_{goal}$.

\subsection{Preliminaries on Inverse Reinforcement Learning}
Given a set of state and action trajectories produced by an MDP, 
the goal of inverse reinforcement learning (IRL) is to infer the MDP's underlying reward function~\cite{ramachandran2007bayesian}.
A set of observed trajectories is given as a dataset~$\mathcal{D} = \{\tau_j\}_{j=1}^N$, where $N$ is the number of trajectories and each~$\tau_j = \{(s_t^j,a_t^j)\}_{t=1}^T$ is a trajectory of state-action pairs of the same length~$T$.
In this work, we use three different IRL algorithms to model the adversary: apprenticeship learning~\cite{abbeel2004apprenticeship}, maximum entropy IRL~\cite{ziebart2008maximum}, and deep maximum entropy IRL~\cite{wulfmeier2015maximum}.

\subsubsection{Apprenticeship Learning}
Apprenticeship learning (AL) uses a feature map $\phi:\mathcal{S}\times \mathcal{A} \to \mathbb{R}^p$  and considers rewards that are linear functions of~$\phi$. With an expert's feature expectation $\mu_E = \frac1N \sum_{\tau_j \in \mathcal{D}}\sum_{t=1}^{T}\phi(s_t^j,a_t^j)$, the goal of AL is to estimate a reward function by first finding a policy $\pi$ whose feature expectation is close to the expert's.
Mathematically, for~$\mu(\pi) = \mathbb{E}[\sum^\infty_{t=0} \gamma^t \phi(s_t,a_t)\mid \pi ] \in \mathbb{R}^p$ and a given~$\epsilon > 0$, AL seeks to find a policy~$\pi$ such
that~$\norm{\mu(\pi) - \mu_E}_2\leq \epsilon$.
Starting from a random initial policy, AL iteratively updates 1) the feature expectation $\mu(\pi)$ 
using the current estimate of the policy~$\pi$, 2) the reward $r$ based on the updated feature expectation~$\mu(\pi)$, and 3) the policy $\pi$ based on updated reward, until $\mu(\pi)$ converges.
We use the projection method from~\cite{abbeel2004apprenticeship} to find a policy, where the detailed algorithm is stated in Algorithm~\ref{alg:projection_method} for completeness. 
At each iteration, the error between the expert feature expectation and the estimated feature expectation is reduced until the error between them is smaller
than a given~$\epsilon > 0$~\cite{abbeel2004apprenticeship}. 

\begin{algorithm}
\caption{Projection method of apprenticeship learning~\cite{abbeel2004apprenticeship}}\label{alg:projection_method}
\begin{algorithmic}
\Require Initial policy $\pi^{(0)}$, expert feature expectation $\mu_E$, error tolerance $\epsilon$
\Ensure $\norm{\mu(\pi) - \mu_E}\leq\epsilon$
\State $\mu^{(0)} \gets \mu(\pi^{(0)})$
\State $\bar\mu^{(0)} \gets \mu^{(0)}$
\State $i \gets 1$
\While{$t > \epsilon$}
\State $w^{(i)} \gets \mu_E - \bar\mu^{(i-1)}$
\State $t^{(i)} \gets \norm{w^{(i)}}_2$
\State $\pi^{(i)}\gets \text{Compute the optimal policy for the MDP using the reward } r(s,a) = (w^{(i)})^T\phi(s,a)$
\State $\mu^{(i)} \gets \mu(\pi^{(i)})$ \Comment{Approximated by Monte Carlo}
\State $\bar\mu^{(i)} \gets \frac{(\mu^{(i)}-\bar\mu^{(i-1)})^T(\mu_E-\bar\mu^{(i-1)})}{(\mu^{(i)}-\bar\mu^{(i-1)})^T(\mu^{(i)}-\bar\mu^{(i-1)})} (\mu^{(i)}-\bar\mu^{(i-1)})$ \\ \hfill \Comment{Projection of $\mu_E$ onto the line through $\bar\mu^{(i-1)}$ and $\mu^{(i)}$}
\State $i \gets i+1$
\EndWhile
\end{algorithmic}
\end{algorithm}

\subsubsection{Maximum Entropy IRL}
Maximum entropy IRL (MaxEnt IRL) uses a feature map $\phi:\mathcal{S} \times \mathcal{A} \to \mathbb{R}^p$ 
to learn a reward function that is linearly parameterized by~$\unknownPara\in\mathbb{R}^p$ 
in the form~$r(s,a) = \unknownPara^T\phi(s,a)$ for each~$(s, a) \in \mathcal{S} \times \mathcal{A}$.
That is, it finds 
the value of~$\unknownPara$ that maximizes the log-likelihood 
of the trajectories in~$\mathcal{D}$ by computing
\begin{align}\label{eq:solve_theta_irl}
    \unknownPara^* &= \underset{\unknownPara}{\operatorname{argmax}}\: \mathcal{L}(\unknownPara) := \sum_{j=1}^N \log\big(P(\observedTraj_j \mid \unknownPara)\big) = \underset{\unknownPara}{\operatorname{argmax}} \sum_{j=1}^N \log\left(\frac{\exp(r_\tau(\observedTraj_j))}{\sum_{j=1}^N \exp(r_\tau(\observedTraj_j))}\right),
\end{align}
where~$r_\tau(\observedTraj_j)=\sum_{t=1}^{T}\tilde{r}(s_t^j,a_t^j)$ is the 
estimated 
reward obtained over the trajectory~$\observedTraj_j$, 
which is 
calculated 
over each point $(s_t^j,a_t^j) \in \observedTraj_j$
with the reward function $\tilde{r}$ that is estimated by IRL.
We use a gradient descent algorithm to solve \eqref{eq:solve_theta_irl}, with 
\begin{equation}
    \nabla \mathcal{L}(\unknownPara) = \mu_E - \sum_{(s^j,a^j)\in\observedTraj_j} 
    \expectedStateVisit(s^j,a^j)\featureVec(s^j,a^j),
\end{equation}
where
$\mu_E = \frac1N \sum_{\tau_j \in \mathcal{D}}\sum_{t=1}^{T}\phi(s_t^j,a_t^j)$ is an expert feature expectation
and $d(s^j,a^j)$ is the expected frequency of visitation of the state-action 
pair~$(s^j,a^j)$. 
For a detailed description of the algorithm to compute $d(s^j,a^j)$, we refer the reader to Algorithm 1 in~\cite{ziebart2008maximum}.

\subsubsection{Deep Maximum Entropy IRL}
Deep maximum entropy IRL (Deep IRL) uses a feature map $\phi:\mathcal{S} \times \mathcal{A} \to \mathbb{R}^p$ 
and uses a deep neural network to estimate a reward function of the form~$r = g(\phi, \omega)$ for
a parameter vector~$\omega \in \mathbb{R}^d$, where $d$ is the number of trainable parameters in the neural network. 
The parameter $\omega$ that maximizes the log-likelihood of the observed trajectories in~$\mathcal{D}$ is learned by solving 
\begin{equation} \begin{aligned}
    \omega^* &= \underset{\omega}{\operatorname{argmax}} \mathcal{L}(\omega):= \sum_{j=1}^N \log\big(P(\observedTraj_j \mid r)\big)  = \underset{\omega}{\operatorname{argmax}} \sum_{j=1}^N \log\big(P(\observedTraj_j \mid  g(\phi,\omega))\big).
\end{aligned}\end{equation}
We apply a gradient descent algorithm to find the optimal parameter $\omega^*$, with
\begin{equation} \begin{aligned} 
    \frac{\partial \mathcal{L}}{\partial \omega} &= \frac{\partial \mathcal{L}}{\partial r}\cdot \frac{\partial r}{\partial \omega} = \big(\mu_E - \mathbb{E}[\mu]\big)\cdot \frac{\partial}{\partial \omega}g(\phi, \omega),
\end{aligned}
\end{equation}
where $\mu_E = \frac1N \sum_{\tau_j \in \mathcal{D}}\sum_{t=1}^{T}\phi(s_t^j,a_t^j)$ is an expert feature expectation
and $\frac{\partial}{\partial \omega}g(\phi, \omega)$ can be obtained with backpropagation.
For a detailed description of the algorithm to compute $\mathbb{E}[\mu]$, we refer the reader to Algorithm 3 in~\cite{wulfmeier2015maximum}.

\begin{remark}
The term $\mu_E$ in all three IRL algorithms represents the expert feature expectation, which is empirically calculated based on the observed trajectories in a given dataset $\mathcal{D}$. Therefore, by influencing the trajectories observed by the adversary,
we can influence the features that they extract from those trajectories.
Since the occupancy measures~$\{x(s,a)\}_{x\in\mathcal{S}, a\in\mathcal{A}}$ are the expectations of the state-action visitations in the observed trajectories, they directly influence these observed trajectories, and thus the feature expectations. Therefore, to implement deception, we manipulate these occupancy measures to steer IRL to desired incorrect beliefs.
\end{remark}

\subsection{Defining Deception}
We seek to use three types of deception, namely ``diversionary'', ``targeted'', and ``equivocal'' deception, which are defined in~\cite{kim2024defining} as follows. 
\begin{definition}[Diversionary Deception]\label{def:diversionary_deception}
A policy~$\policy_d$ is
diversionary deceptive if the observer's inference of 
a system's goal
state is incorrect in any way. 
\end{definition} 
\begin{definition}[Targeted Deception]\label{def:targeted_deception}
A policy~$\policy_{d}$ is targeted
deceptive if the observer infers a particular set of user-specified incorrect
states as the set of goal states. 
\end{definition} 
\begin{definition}[Equivocal Deception]\label{def:equivocal_deception}
A policy~$\policy_{d}$ is equivocal deceptive if the observer infers that a real set of 
goal states and a user-specified incorrect set of
goal states are equally likely to explain the system's behavior. 
\end{definition} 

In Definition~\ref{def:goal_states}, we defined the goal states as states with a maximum reward greater than or equal to the reward at all non-goal states. Consequently, by manipulating reward information, we can conceal the goal states.

\subsection{Problem Statements}
A policy satisfying Definition~\ref{def:diversionary_deception}, \ref{def:targeted_deception}, or~\ref{def:equivocal_deception}
may not achieve the optimal total accumulated reward simply because the need for
deception introduces considerations other than optimality of performance. 
To formally establish the relationship between deceptiveness and performance, we will solve the following problems:

\begin{problem} \label{prob:performance_loss}
    For a policy satisfying Definition~\ref{def:diversionary_deception}, \ref{def:targeted_deception}, or~\ref{def:equivocal_deception}, analytically bound the performance loss 
    by examining the decrease in the expected total accumulated reward that
    is caused by deception. 
\end{problem}
\begin{problem} \label{prob:validation}
    Empirically analyze the deceptiveness and performance of the policies that satisfy Definition~\ref{def:diversionary_deception}, \ref{def:targeted_deception}, or~\ref{def:equivocal_deception}.
\end{problem}

Problems~\ref{prob:performance_loss} and~\ref{prob:validation} are 
the focus of the rest of the paper.

\section{Deception Implementations and Performance Guarantees} \label{sec:method}
In this section, we solve Problem~\ref{prob:performance_loss} and bound
the loss in total accumulated reward that is induced by each type of deception. 
For simplicity, 
we refer to the total accumulated reward as ``revenue''. 
Given an MDP, for any policy $\policy$,
the revenue can be calculated as 
\begin{equation} \label{eq:revenue}
R_\policy = \sum_{s \in \mathcal{S}}\sum_{a\in\mathcal{A}} r(s,a)x_\policy(s,a) 
\end{equation}
as noted in~\cite{puterman2014markov}. 
Before we bound the revenue loss induced by deception, 
we first formally define the
notion of revenue loss that we consider.


\begin{definition}[Revenue Loss] \label{def:revenue_loss}
Let $\{x^*(s,a)\}_{s\in\mathcal{S},a\in\mathcal{A}}$ be the set of optimal occupancy measures obtained from solving Optimization Problem~\ref{op:dualLP_with_task_constraint} and let $R^* = \sum_{s \in \mathcal{S}}\sum_{a\in\mathcal{A}_s} r(s,a)x^*(s,a)$ be the corresponding revenue. Then, 
the revenue loss associated with a non-optimal policy~$\pi$ is
\begin{equation} \label{eq:revenue_loss}
    L_{\policy} = \frac{R^*-R_\policy}{R^*}.
\end{equation}
\end{definition}

In words, given a sub-optimal policy~$\pi$ 
we define the revenue loss~$L_{\pi}$ as the fraction of the optimal reward 
that is lost when following~$\pi$.

\subsection{Implementation and Performance Guarantees for Diversionary Deception}
To implement diversionary deception in the sense of Definition \ref{def:diversionary_deception}, we modify Optimization Problem~\ref{op:dualLP_with_task_constraint} to obtain Optimization Problem~\ref{op:diversionary_optimization_problem}, which 
also
appears in~\cite{kim2024defining}:
\begin{oproblem}\label{op:diversionary_optimization_problem}
\begin{equation}
\begin{aligned}
\underset{ x\in \mathbb{R}^{|\mathcal{S}||\mathcal{A}|}}{\operatorname{maximize}} \,\,  & \sum_{s \in \mathcal{S}} \sum_{a \in \mathcal{A}} \left[ r(s,a) x(s,a) + \beta \big(x(s,a) - x^*(s,a)\big)^2 \right] \\
\textnormal{ subject to } \,\, 
&\sum_{a \in \mathcal{A}} x(j,a) - \gamma \sum_{s \in \mathcal{S}} \sum_{a \in \mathcal{A}} \mathcal{T}(s,a,j) x(s,a)  = \idist (j) \quad \textnormal{ for all } j \in \mathcal{S}\\
&\sum_{q \in \mathcal{S}_{goal}}  \sum_{s\in \mathcal{S}} \sum_{a \in \mathcal{A}}  \mathcal{T}(s,a,q) x(s,a) \geq v_{reach} \\
& x(s,a) \geq 0 \quad \textnormal{ for all } s\in \mathcal{S},a\in \mathcal{A}, 
\end{aligned}
\end{equation} 
\end{oproblem}
where~$x^*$ is the solution to Optimization Problem~\ref{op:dualLP_with_task_constraint}. 

To obtain Optimization Problem~\ref{op:diversionary_optimization_problem}, a weighted quadratic term is added to the cost function of Optimization Problem~\ref{op:dualLP_with_task_constraint} 
in order 
to maximize the difference 
between the solution to Optimization Problem~\ref{op:diversionary_optimization_problem} 
and the original, non-deceptive occupancy measures $x^*$. 
The weighting parameter~$\beta>0$ encodes the relative importance of deception versus performance.
A higher~$\beta$ yields a more deceptive policy and a smaller~$\beta$ value places greater emphasis on performance. 


Optimization Problem~\ref{op:diversionary_optimization_problem} problem maximizes a strongly convex function over a compact, convex set. Hence
its solution exists and is attained at some point on the boundary of the
feasible region~\cite{boyd2004convex}. 
In this problem, the parameter $\beta > 0$ 
is chosen by the user.
To provide guidelines for the selection of $\beta$, we next bound the worst-case loss in revenue 
in terms of~$\beta$, and this theorem
solves Problem~\ref{prob:performance_loss} for
diversionary deception. 

\begin{theorem}\label{thm:bound_div}
Given a Multi-agent Markov Decision Process (MMDP) $\mathcal{M} = (\mathcal{S}, \mathcal{A}, r, \gamma, \mathcal{T}, \idist)$, 
let~$x^* = \{x^*(s,a)\}_{s\in\mathcal{S},a\in\mathcal{A}}$ be the optimal occupancy measure and let~$R^*$ be the revenue corresponding to~$x^*$.
When using the policy~$\policy_d$ obtained by solving Optimization Problem~\ref{op:diversionary_optimization_problem}, the loss in revenue in the sense of Definition~\ref{def:revenue_loss} is bounded via
\begin{equation}\label{eq:div_bound}
    L_{\policy_d}\leq\frac{\beta}{R^*}\left(\sum_{s \in\mathcal{S}}\sum_{a\in\mathcal{A}} x^*(s,a)^2 + (1-\gamma)^{-2}\right),
\end{equation}
where $\beta > 0$ is the deception parameter in Optimization Problem~\ref{op:diversionary_optimization_problem}.
\end{theorem}
\begin{proof}
See Appendix~\ref{prf:bound_div}.
\end{proof}

Theorem~\ref{thm:bound_div} reveals that the bound on the revenue loss increases linearly with the deception parameter $\beta$, which enables the design of $\beta$ based on the maximum allowable revenue loss. 
The right-hand term in~\eqref{eq:div_bound} decreases with smaller values of~$\gamma$. With the interpretation that a larger~$\gamma$ encodes a
greater importance of future rewards, Theorem~\ref{thm:bound_div} implies that deception induces more performance loss when greater emphasis is placed on future rewards.
For any value of~$\gamma$, the bound in Theorem~\ref{thm:bound_div} 
 facilitates the deployment 
 of diversionary deception in scenarios where assuring at least a certain level of performance is critical. 


\subsection{Implementation and Performance Guarantees for Targeted Deception}
To achieve targeted deception in the sense of Definition~\ref{def:targeted_deception}, we leverage Optimization Problem 5 from~\cite{kim2024defining}, which we state next for completeness.
\begin{oproblem}\label{op:targeted_optimization_problem}
\begin{equation}
\begin{aligned}
\underset{ x\in \mathbb{R}^{|\mathcal{S}||\mathcal{A}|}}{\operatorname{maximize}} \,\, &  \sum_{s \in \mathcal{S}} \sum_{a \in \mathcal{A}}  \left[ r(s,a) x(s,a) - \beta \big(x(s,a) - x_{tar}(s,a)\big)^2 \right]\\
\textnormal{subject to } &\sum_{a \in \mathcal{A}} x(j,a) - \gamma \sum_{s \in \mathcal{S}} \sum_{a \in \mathcal{A}} \mathcal{T}(s,a,j) x(s,a) = \idist (j) \quad \textnormal{ for all } j \in \mathcal{S}\\
 &\sum_{q \in \mathcal{S}_{goal}}  \sum_{s\in \mathcal{S}} \sum_{a \in \mathcal{A}}  \mathcal{T}(s,a,q) x(s,a) \geq v_{reach} \\
 & x(s,a) \geq 0 \quad \textnormal{ for all } s\in \mathcal{S},a\in \mathcal{A},
\end{aligned}
\end{equation} 
\end{oproblem}
where $x_{tar}=\{x_{tar}(s,a)\}_{s\in\mathcal{S}, a\in\mathcal{A}}$ are user-specified target occupancy measures that encode false goal states, which we aim to lead an observer to
infer as the real goal states.

To obtain Optimization Problem~\ref{op:targeted_optimization_problem}, a weighted quadratic term is subtracted from the cost function of Optimization Problem~\ref{op:dualLP_with_task_constraint}, which seeks to minimize the difference 
between the solution to Optimization Problem~\ref{op:targeted_optimization_problem} and 
the target occupancy measures $x_{tar}$. 
Optimization Problem~\ref{op:targeted_optimization_problem} is a maximization of a strongly concave function over a convex set.
Therefore, it has a unique solution, and it can be found efficiently via convex quadratic programming.
As in Optimization Problem~\ref{op:diversionary_optimization_problem},
the weighting parameter~$\beta > 0$ governs the trade-off between deception and performance, where a larger~$\beta$ leads to a more deceptive policy and a smaller~$\beta$ prioritizes performance. To provide guidelines for selecting~$\beta$, we next establish performance guarantees 
for targeted deception as a function of $\beta$.

\begin{theorem}\label{thm:bound_tar}
Given a Multi-agent Markov Decision Process (MMDP) $\mathcal{M} = (\mathcal{S}, \mathcal{A}, r, \gamma, \mathcal{T}, \idist)$, let
$x^* = \{x^*(s,a)\}_{s\in\mathcal{S},a\in\mathcal{A}}$ be the set of optimal occupancy measures and let $R^*$ be the revenue corresponding to $x^*$.
When employing the policy~$\policy_d$ obtained by solving Optimization Problem~\ref{op:targeted_optimization_problem}, the loss in revenue in the sense of Definition~\ref{def:revenue_loss} is bounded via 
\begin{equation} \label{eq:tar_bound}
\begin{aligned}
    L_{\policy_d}\leq  \frac{\beta}{R^*} \left[\sum_{s \in\mathcal{S}}\sum_{a\in\mathcal{A}} \Big(x^*(s,a)^2 - 2x_{tar}(s,a)x^*(s,a)\Big) -\frac{(1-\gamma)^{-2}}{|\mathcal{S}||\mathcal{A}|}  +2(1-\gamma)^{-1}\max_{s\in\mathcal{S},a\in\mathcal{A}}x_{tar}(s,a)\right],
\end{aligned}
\end{equation}
where $x_{tar}= \{x_{tar}(s,a)\}_{s\in\mathcal{S},a\in\mathcal{A}}$ is a user-specified set of target occupancy measures and $\beta > 0$ is the deception parameter in Optimization Problem~\ref{op:targeted_optimization_problem}.
\end{theorem}
\begin{proof}
    See Appendix~\ref{prf:bound_tar}. 
\end{proof}

Theorem~\ref{thm:bound_tar} shows that targeted deception decreases revenue the most when less emphasis is placed on future rewards, 
which is encoded by having a smaller~$\gamma$. Furthermore, Theorem~\ref{thm:bound_tar} indicates that the design of~$x_{tar}$ also impacts performance.
To implement deception, the target occupancy measures $x_{tar}$ should be designed to assign larger values to the fake goal states instead of the real goal states. 
However, a larger maximum value of~$x_{tar}$ yields worse performance in the worst case, as indicated by the~$\max$ term in~\eqref{eq:tar_bound},
and there are clear tradeoffs between deception and performance that come from the maximum value taken by~$x_{tar}(s, a)$. 
Additionally, if the values of~$x_{tar}$ become closer to those of~$x^*$ 
while keeping the sum of the squares of the values of~$x_{tar}$ constant, then 
the worst-case performance improves.
After $\gamma$ and $x_{tar}$ are fixed, the bound on the revenue loss increases linearly with the deception parameter $\beta$, which enables a user to easily choose~$\beta$ based
on a maximum allowable loss in performance.

\subsection{Implementation and Performance Guarantees for Equivocal Deception}

To implement equivocal deception in the sense of Definition \ref{def:equivocal_deception}, we leverage Optimization Problem~6 in~\cite{kim2024defining}, which is stated below:
\begin{oproblem}\label{op:equivocal_optimization_problem}
\begin{equation}
\begin{aligned}
\underset{ x\in \mathbb{R}^{|\mathcal{S}||\mathcal{A}|}}{\operatorname{maximize}} \,\, &\sum_{s \in \mathcal{S}} \sum_{a \in \mathcal{A}}  r(s,a) x(s,a)\\
\textnormal{ subject to } \,\,  & \sum_{a \in \mathcal{A}} x(j,a) - \gamma \sum_{s \in \mathcal{S}} \sum_{a \in \mathcal{A}} \mathcal{T}(s,a,j) x(s,a) = \idist (j) \quad \textnormal{ for all } j \in \mathcal{S}\\
&\sum_{q \in \mathcal{S}_{goal}}  \sum_{s\in \mathcal{S}} \sum_{a \in \mathcal{A}}  \mathcal{T}(s,a,q) x(s,a) \geq v_{reach} \\
&  \sum_{s \in \mathcal{S}_{goal}} \sum_{a \in \mathcal{A}} x(s,a) =  \sum_{s \in \mathcal{S}_{decoy}} \sum_{a \in \mathcal{A}} x(s,a) \\
& x(s,a) \geq 0 \quad \textnormal{ for all } s\in \mathcal{S},a\in \mathcal{A}, 
\end{aligned}
\end{equation} 
\end{oproblem}
where~$\mathcal{S}_{goal}$ is the set of real goal states and $\mathcal{S}_{decoy}$ is the set of decoy goal states, which is introduced to mislead the adversary. 

To obtain Optimization Problem~\ref{op:equivocal_optimization_problem}, an equality constraint is added to Optimization Problem~\ref{op:dualLP_with_task_constraint} to enforce that the expected visitations to $\mathcal{S}_{goal}$ and $\mathcal{S}_{decoy}$ are equal. 
Compared to Optimization Problem~\ref{op:dualLP_with_task_constraint}, this additional constraint shrinks the feasible region of the problem 
and may make the problem infeasible under some choices of problems and parameters.
Moreover, this problem does not provide a means for adjusting the strength of deception.
Therefore, to guarantee the existence of the solution and enable control over the strength of deception, we modify Optimization Problem~\ref{op:equivocal_optimization_problem} by incorporating the new constraint into the objective function:

\begin{oproblem}\label{op:modified_equivocal_optimization_problem}
\begin{equation}
\begin{aligned}
\underset{ x\in \mathbb{R}^{|\mathcal{S}||\mathcal{A}|}}{\operatorname{maximize}} \,\, & \sum_{s \in \mathcal{S}} \sum_{a \in \mathcal{A}}  r(s,a) x(s,a)  - \beta\left( \sum_{s \in \mathcal{S}_{goal}} \sum_{a \in \mathcal{A}} x(s,a) -  \sum_{s \in \mathcal{S}_{decoy}} \sum_{a \in \mathcal{A}} x(s,a) \right)^2\\
\textnormal{ subject to } \,\,  & \sum_{a \in \mathcal{A}} x(j,a) - \gamma \sum_{s \in \mathcal{S}} \sum_{a \in \mathcal{A}} \mathcal{T}(s,a,j) x(s,a) = \idist (j) \quad \textnormal{ for all } j \in \mathcal{S}\\
&\sum_{q \in \mathcal{S}_{goal}}  \sum_{s\in \mathcal{S}} \sum_{a \in \mathcal{A}}  \mathcal{T}(s,a,q) x(s,a) \geq v_{reach} \\
& x(s,a) \geq 0 \quad \textnormal{ for all } s\in \mathcal{S},a\in \mathcal{A}.
\end{aligned}
\end{equation} 
\end{oproblem}

The second term in the objective function minimizes the squared difference of the sum of occupancy measures in $\mathcal{S}_{goal}$ and $\mathcal{S}_{decoy}$.
Similar to Optimization Problems~\ref{op:diversionary_optimization_problem} and~\ref{op:targeted_optimization_problem}, the weighting parameter~$\beta>0$ encodes 
the trade-off between deception and performance.
If the accumulated reward, i.e., the first term of the objective function, can be maximized while satisfying $\sum_{s \in \mathcal{S}_{goal}} \sum_{a \in \mathcal{A}} x(s,a) =  \sum_{s \in \mathcal{S}_{decoy}} \sum_{a \in \mathcal{A}} x(s,a)$, then the deception constraint in Optimization Problem~\ref{op:equivocal_optimization_problem} is always satisfied.
Otherwise, deception is prioritized by a larger value of~$\beta$. 
The objective function in Optimization Problem~\ref{op:modified_equivocal_optimization_problem} is a strongly concave quadratic function when~$\beta > 0$.
Optimization Problem~\ref{op:modified_equivocal_optimization_problem} has the same feasible region as Optimization Problem~\ref{op:dualLP_with_task_constraint}, which is known to be non-empty, and therefore a solution exists to Optimization Problem~\ref{op:modified_equivocal_optimization_problem}.
Since it maximizes a strongly concave function over a convex, compact set, its solution can be computed efficiently with convex qudratic programming.

In this problem, the parameter $\beta > 0$ is chosen by the user depending on the requirements for deception and performance, where larger value of $\beta$ emphasizes deception and a smaller value of $\beta$ prioritize the performance.
By providing a bound on the revenue loss with respect to $\beta$, the following theorem gives practical insight into the selection of $\beta$.
\begin{theorem}\label{thm:bound_equ}
Given a Multi-agent Markov Decision Process (MMDP) $\mathcal{M} = (\mathcal{S}, \mathcal{A}, r, \gamma, \mathcal{T}, \idist)$, let~$x^* = \{x^*(s,a)\}_{s\in\mathcal{S},a\in\mathcal{A}}$ be the optimal occupancy measure and let~$R^*$ be the revenue corresponding to~$x^*$. Define~$\mathcal{S}_{goal}$ to be the set of real goal states and $\mathcal{S}_{decoy}$ to be the set of decoy goal states. 
When using the policy~$\policy_d$ obtained by solving Optimization Problem~\ref{op:modified_equivocal_optimization_problem}, the loss in revenue in the sense of Definition~\ref{def:revenue_loss} is bounded via
\begin{equation}\label{eq:equ_bound}
\begin{aligned}
    L_{\policy_d} \leq \frac{\beta}{R^*}\left(\sum_{s\in\mathcal{S}_{goal}}\sum_{a\in\mathcal{A}}x^*(s, a) - \sum_{s\in\mathcal{S}_{decoy}}\sum_{a\in\mathcal{A}}x^*(s, a)\right)^2,
\end{aligned} 
\end{equation}
where~$\beta > 0$ is the deception parameter in Optimization Problem~\ref{op:modified_equivocal_optimization_problem}.
\end{theorem}
\begin{proof}
See Appendix~\ref{prf:bound_equ}.
\end{proof}
Theorem~\ref{thm:bound_equ} provides a guideline for selecting~$\beta$ by establishing the linear dependence of the worst-case revenue loss on~$\beta$.
According to Theorem~\ref{thm:bound_equ}, the worst-case revenue loss decreases as the optimum~$\{x^*(s,a)\}_{s\in\mathcal{S}, a\in \mathcal{A}}$ takes closer values 
on the sets~$\mathcal{S}_{goal}$ and~$\mathcal{S}_{decoy}$. If the difference in visitation between~$\mathcal{S}_{goal}$ and~$\mathcal{S}_{decoy}$ is zero for~$x^*$, 
then the performance loss is guaranteed to be zero. Problems that satisfy this condition are inherently amenable to
implementing equivocal deception. 

Theorems~\ref{thm:bound_div}, \ref{thm:bound_tar}, and~\ref{thm:bound_equ} together solve Problem~\ref{prob:performance_loss}. 
\section{Numerical Results} \label{sec:result}
In this section, we solve Problem~\ref{prob:validation}.
We use our techniques for deceptive policy optimization 
in the setting of the moving target defense (MTD) problem from~\cite{zheng2019markov}, which can be generalized to real-world network defense problems, such as DDoS defense~\cite{WANG201410, 6614155} and protecting smart grid~\cite{6175633} or vehicle networks~\cite{8424662}.
By extending the problem to a multi-agent setting, the objective is to hide ``real agent(s)'' whose task must be completed 
among a collection of ``decoy agents''
that are part of the system only to aid in the implementation
of deception. We seek to conceal the real agent(s)
while maintaining strong performance in the sense of still maximizing revenue as much as possible. 
To that end, we empirically 
evaluate the extent to which the policies obtained from Optimization Problems~\ref{op:diversionary_optimization_problem}, ~\ref{op:targeted_optimization_problem}, and~\ref{op:modified_equivocal_optimization_problem} achieve both deception and strong performance simultaneously. 
\footnote{The software and data used to generate these results is available at: \href{https://anonymous.4open.science/r/DeceptionMTD-2980}{https://anonymous.4open.science/r/DeceptionMTD-2980}.}.

We use the MMDP~$\mathcal{M} = (\mathcal{S}, \mathcal{A}, r, \gamma, \mathcal{T}, \idist)$ to model the system, where there are~$n$ real agents and~$m$ decoy agents in a network.
Each agent has the state space~$\mathcal{S}^i = \{\texttt{N}, \texttt{T}, \texttt{E}, \texttt{B}\}$ and action space~$\mathcal{A}^i = \{\texttt{wait}, \texttt{defend}, \texttt{reset}\}$.
The states \texttt{N}, \texttt{T}, \texttt{E}, and \texttt{B} corresponds to the system conditions ``normal'', ``targeted'', ``exploited'', and ``breached'', respectively,
indicating a degree of compromise by an adversary. The state ``normal''  indicates the absence of any security threats, ``targeted'' signifies that the agent is under attack but not yet compromised, ``exploited''  refers to a state where vulnerabilities 
in the agent
have been used by an adversary, 
and ``breached'' indicates a full compromise of the agent
by an adversary.

\begin{figure}[ht]
\begin{center}
\includegraphics[width=4.0 in]{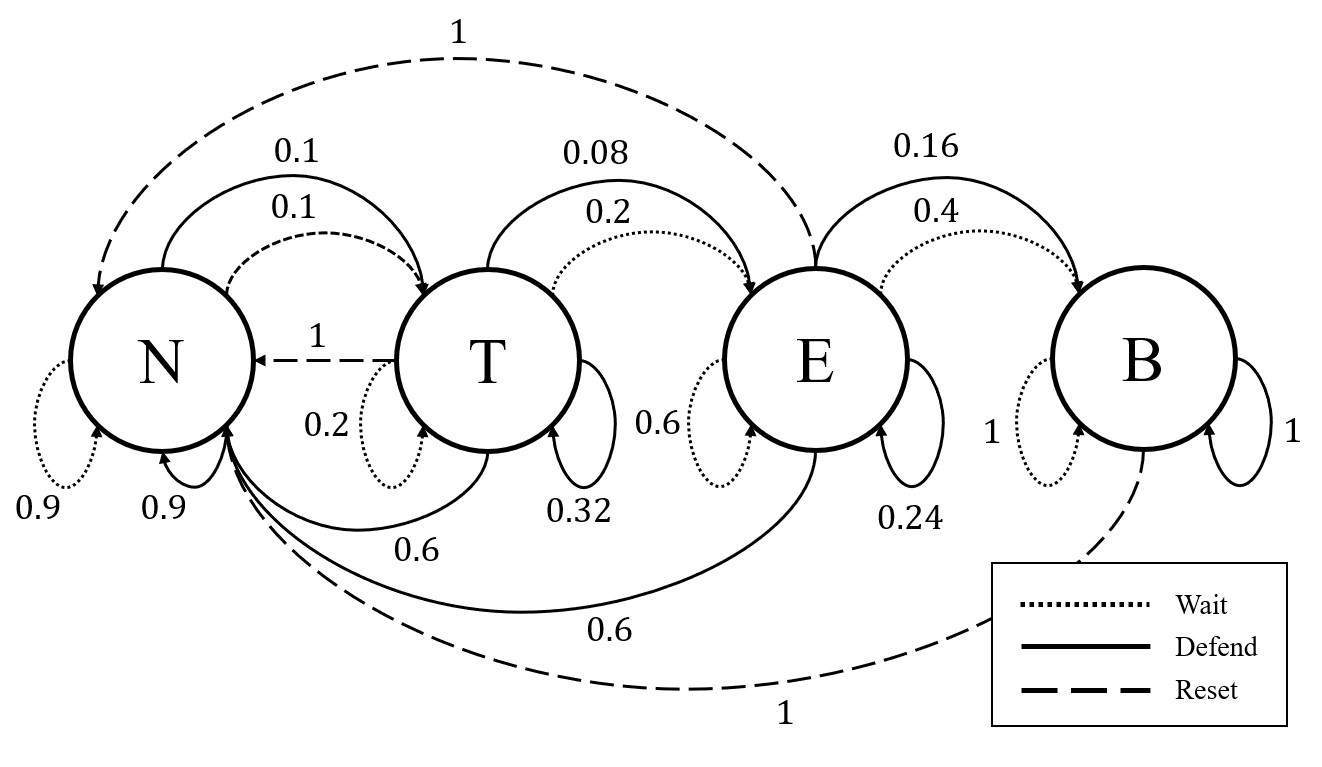}
\end{center}
\caption{Transition probabilities of each agent in the MTD problem. }
\label{fig:MTD_transition_probability}
\end{figure}

\begin{table}[ht]
\caption{Rewards for a real agent at each state-action pair in the MTD problem.}
\vspace{.1 in}
\centering
\begin{tabular}{c|lll}\toprule
        State & \texttt{wait} & \texttt{defend} & \texttt{reset}\\ \midrule
        \texttt{N} & $R$ & $R-C_D$ & $R - C_R$\\
        \texttt{T} & $R-C_T$ & $R-C_D -C_T$ & $R - C_R$\\
        \texttt{E} & $R-C_E$ & $R-C_D -C_E$ & $R - C_R$\\
        \texttt{B} & $R-C_B$ & $R-C_D -C_B$& $R - C_R$ \\ \bottomrule
\end{tabular}
\label{tab:rewards}
\end{table}

All agents start in state \texttt{N}. The transition probabilities of each agent are defined as in Figure~\ref{fig:MTD_transition_probability}~\cite{kim2024defining}. 
Table~\ref{tab:rewards} lists the reward values that $r^i$ takes for real agent~$i$
in each state-action pair. To reflect that decoy agents do not play a significant role in the 
performance of the 
system, their rewards are set to~$1\%$ of the reward values of the real agents
that are shown in Table~\ref{tab:rewards}. 

We adopt parameters in Table~\ref{tab:rewards} from~\cite{zheng2019markov}. 
We use the baseline reward~$R = 10$ that is earned at each timestep,
along with various costs (which are negative rewards): $C_T = 0.1$ for being targeted (occupying state \texttt{T}),~$C_E = 3$ for being exploited (occupying state \texttt{E}), $C_B = 4$ for being breached (occupying state \texttt{B}), $C_D = 5$ for taking the \texttt{defend} action, and~$C_R = 20$ for taking the \texttt{reset} action.

We use three types of IRL to model the adversary: 
AL, MaxEnt IRL, and Deep IRL. The objective of the adversary is to identify the real agent(s) among the full collection of agents. After estimating the agents' rewards using IRL, 
the adversary infers the real agent(s) as the agent(s) that have
higher reward in state ``\texttt{N}" and lower reward in state ``\texttt{B}".
If the adversary makes an incorrect estimate of the 
system's 
rewards, 
then
they will infer incorrect goal states, 
which will result in 
incorrect inferences of the identities of 
the real agent(s).
Hence, to conceal the real agent from the adversary, we compute deceptive policies to influence the reward estimated by IRL.

From the adversary's perspective, they infer the real agent based on the likelihood of each agent being the real agent, computed as
\begin{equation}\label{eq:likelihood}
    L^i = \frac{\exp{(D^i/\lambda)}}{\sum_{i=1}^{n+m} \exp{(D^i/\lambda)}},
\end{equation}
with some~$D^i$ that measures the importance of agent~$i$, and some temperature parameter~$\lambda$.
For IRL's estimate of the reward function~$\estReward$,
we calculate~$D^i$ by leveraging the deception metric proposed in~\cite{kim2024defining}, namely
\begin{equation}
D^{i} = {\sum_{y^i\in\mathcal{S}^i}w(y^i)\estTotalReward^i(y^i)},
\end{equation}
for $i \in \{1, \dots, n+m\}$,
where~$w(y^i)\geq0$ are weights for each local state~$y^i$ 
and where 
\begin{equation}
\estTotalReward^i(y^i)=\sum_{\jointState=(s^1,\dots,s^i=y^i,\dots,s^{n+m})\in\jointStateSpace}
\sum_{\jointAction\in\jointActionSpace}\estReward(\jointState,\jointAction)
\end{equation}
is the estimated marginal reward for agent $i$ at local state~$y^i$. 
 For weights of~$D^i$, we used $w(\texttt{N}) = 1$, $w(\texttt{T}) = 0.75$, $w(\texttt{E}) = 0.5$, and $w(\texttt{B}) = 0.25$, reflecting that each agent may be considered important if it has (i) higher rewards in state \texttt{N} and (ii) lower rewards in state \texttt{B}. 
Intuitively, this metric represents the estimated importance of each agent based on the estimated rewards, where states with higher rewards are considered more important. 
A larger value of $D^i$ will result in a larger value of $L^i$, and the adversary identifies the real agent as the one with the largest value of~$L^i$.

In this example, we consider 5 agents, where the number of real and decoy agents varies in each scenario.
We computed deceptive policies by solving Optimization Problems~\ref{op:diversionary_optimization_problem}, ~\ref{op:targeted_optimization_problem}, and~\ref{op:modified_equivocal_optimization_problem}, each with $v_{reach} = 7$, which are formulated over~$12^5 = 248,832$ joint state-action pairs.
Building on Definitions~\ref{def:diversionary_deception}, ~\ref{def:targeted_deception}, and~\ref{def:equivocal_deception} in the context of this problem, we say that a policy is diversionary deceptive if the adversary misidentifies the real agent in any way, 
it is
targeted deceptive if the adversary infers that the particular 
choice of decoy agent is the real agent,
and it is equivocal deceptive if the adversary infers that the specific decoy agent is as important as the real agent.
In the next two subsections, we demonstrate the deceptiveness and performance of policies obtained from Optimization Problems~\ref{op:diversionary_optimization_problem},~\ref{op:targeted_optimization_problem}, and~\ref{op:modified_equivocal_optimization_problem}, then validate Theorems~\ref{thm:bound_div},~\ref{thm:bound_tar}, and~\ref{thm:bound_equ} by comparing empirically obtained revenue values
to theoretically predicted values. 


\subsection{Deceptiveness}
In this section, we evaluate the deceptiveness of the deceptive policy $\pi_d$ obtained by solving Optimization Problem~\ref{op:diversionary_optimization_problem},~\ref{op:targeted_optimization_problem}, or~\ref{op:modified_equivocal_optimization_problem} in adversarial settings.
In the context of MTD and IRL, the adversary infers that the agent with the maximum $L^i$ value is the real agent.
To measure the results of deception, we define the adversary's utility under deception. 
After calculating~$L^i$ based on the estimated rewards, the adversary attacks the agent that has the maximum value of~$L^i$. As a result, the adversary gets utility $1$ if it attacked the true target, gets utility $-1$ if attacked the ``target decoy agent" we selected, and gets utility $0$ otherwise.
In Table~\ref{tab:diversionary},~\ref{tab:targeted}, and~\ref{tab:equivocal}, we evaluate this utility of the adversary after 100 runs with different values of the parameter~$\beta$.
In Figure~\ref{fig:diversionary},~\ref{fig:targeted} and~\ref{fig:equivocal}, we show the adversary's estimates of~$L^i$'s using three different IRL algorithms: AL, MaxEnt IRL, and Deep IRL.

\subsubsection{Diversionary Deception}
From Definition~\ref{def:diversionary_deception}, a policy that is diversionary deceptive seeks to lead an adversary to infer that any non-goal state is a goal state when attempting to learn the rewards of a system.
In the context of MTD, we seek to make the adversary infer that any decoy agent is the real agent.
We therefore want any decoy agent to yield the highest~$L^i$. 
To verify this goal, we considered 2 scenarios that have different ratios of real and decoy agents, where we have one real and four agents decoy in scenario 1, and two real and three decoy agents in scenario 2.

\begin{figure}
    \centering
    \includegraphics[width=01\linewidth]{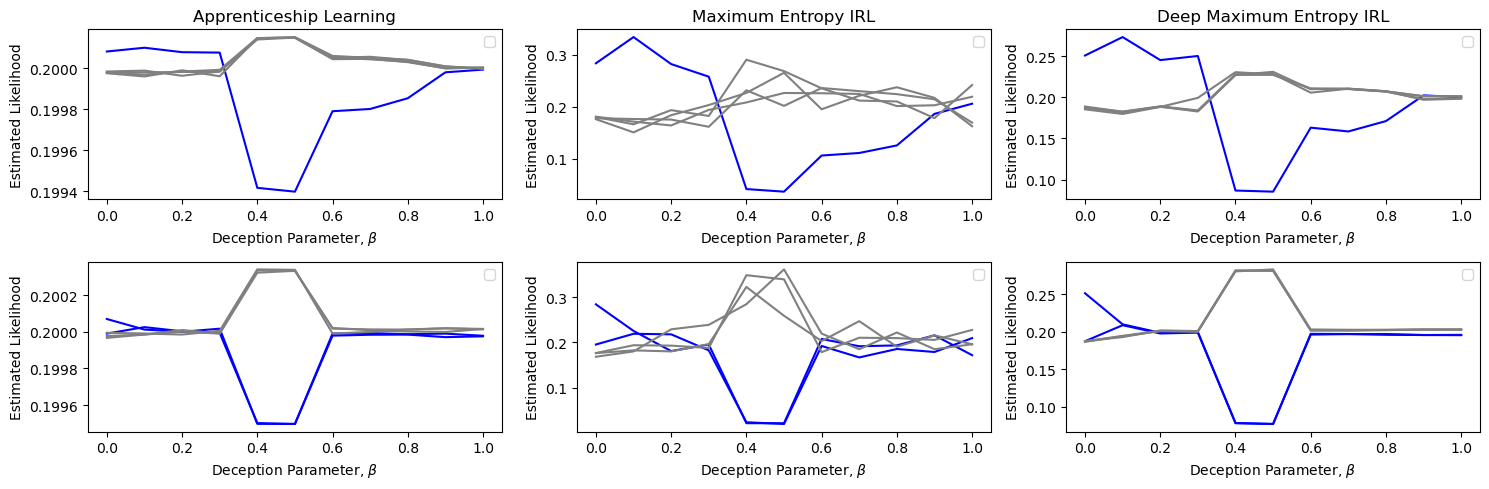}
    \caption{The adversary's estimate of~$L^i$ of each agent 
    under diversionary deception with three types of IRL. Each plot shows results averaged over~$100$ runs. We used $\textit{number of trajectories} = 500$ and $\textit{length of trajectories} = 500$ for all simulations, $\textit{epoch} = 20$, $\textit{learning rate} = 0.5$ for MaxEnt IRL, and $\textit{epoch} = 50$, $\textit{learning rate} = 0.1$ and two fully connected hidden layers with dimension $64$ and $32$ for Deep IRL. The blue line(s) represents the real agent(s), while the gray lines represent decoy agents.
    When~$\beta = 0$, no deception is implemented, and the adversary correctly infers that the real agent(s) is the most important. However, with~$\beta \geq 0.4$, the adversary's inference of the most important agent will be one of the decoy agents in average, confirming that the policy satisfies the definition of diversionary deception.}
    \label{fig:diversionary}
\end{figure}
\begin{table}
    \centering
    \begin{tabular}{c|c|c|cccccccccc}
    \hline \hline
          \multicolumn{2}{|c|}{Deception Parameter $\beta$}& 0&0.1&0.2&0.3&0.4&0.5&0.6&0.7&0.8&0.9&1.0 \\
         \hline
         \multirow{6}{4.5em}{Scenario 1} 
         & Real &100&100&100&100& 0&  0&  5&  2&  0& 42&22\\
         & Decoy 1&0&0&0&0& 34&  17&  26&  18&  18& 30&28\\
         & Decoy 2&0&0&0&0& 21&  13&  22&  17&  22& 5&12\\
         & Decoy 3&0&0&0&0& 20&  8&  23&  30&  31& 3&8\\
         & Decoy 4&0&0&0&0& 25&  62&  24&  33&  29& 20&30\\
         \cline{2-13}
         & \textbf{Adversary's Utility}&\textbf{100}&\textbf{100}&\textbf{100}&\textbf{100}& \textbf{0}&  \textbf{0}&  \textbf{5}&  \textbf{2}&  \textbf{0}& \textbf{42}&\textbf{22}\\
         \hline
         \multirow{6}{4.5em}{Scenario 2} 
         & Real 1 &100&45&19&16& 0&  0&  16&  7&  3& 1&1\\
         & Real 2&0&47&14&18& 0&  0&  12&  7&  2& 3&3\\
         & Decoy 1&0&2&23&25& 35&  36&  31&  41&  44& 33&32\\
         & Decoy 2&0&3&21&20& 38&  28&  20&  26&  35& 34&34\\
         & Decoy 3&0&3&23&21& 27&  36&  21&  19&  16& 29&30\\
         \cline{2-13}
         & \textbf{Adversary's Utility}& \textbf{100}& \textbf{92}&\textbf{33}& \textbf{34}&  \textbf{0}&   \textbf{0}& \textbf{28}& \textbf{14}&  \textbf{5}&   \textbf{4}&  \textbf{4}\\
         \hline\hline
    \end{tabular}
    \caption{Adversary's estimation of the real agent and total utility when using Deep IRL under diversionary deception}
    \label{tab:diversionary}
\end{table}

Figure~\ref{fig:diversionary} shows that this goal is attained in both scenarios for larger values of the deception parameter~$\beta$. 
With~$\beta \geq 0.4$, we see that the real agent has a lower likelihood than at least one of the decoy agents when using any of the considered IRL algorithms, implying that there exists a level of deception capable of simultaneously deceiving all of the tested algorithms.
Table~\ref{tab:diversionary} shows the estimation of the real agent and the resulting utility when its using Deep IRL.
In both scenarios, with $\beta = 0$ (when there is no deception), the adversary always find out the real agent(s) and successfully attacks them. However, with $\beta \simeq 0.4$, the adversary never attacks the real agent(s), while larger values of $\beta$ tends to show mixed likelihood, making all agents indistinguishable in their importance. 
This implementation provides diversionary deception with $\beta \geq 0.4$ by ensuring that the adversary draws an incorrect conclusion about the identity of the real agent(s).

\subsubsection{Targeted Deception}

From Definition~\ref{def:targeted_deception}, a policy that is targeted deceptive seeks to lead an adversary to a specific, user-specified false belief. In the context of MTD, we seek to lead an adversary to infer that one of the target decoy agents is the real agent. 
 To verify this goal, we consider 2 scenarios that have different ratios of real and target decoy agent(s), where we have one real and one target decoy agents in scenario 1, and one real and two target decoy agents in scenario 2.
To achieve targeted deception, we aim to manipulate the occupancy measure $\{x(s,a)\}_{s\in\mathcal{S},a\in \mathcal{A}}$ to influence the metric~$L^i$ so that the largest value of $L^i$ is obtained with the target decoy agent(s).

\begin{figure}
    \centering
    \includegraphics[width=1\linewidth]{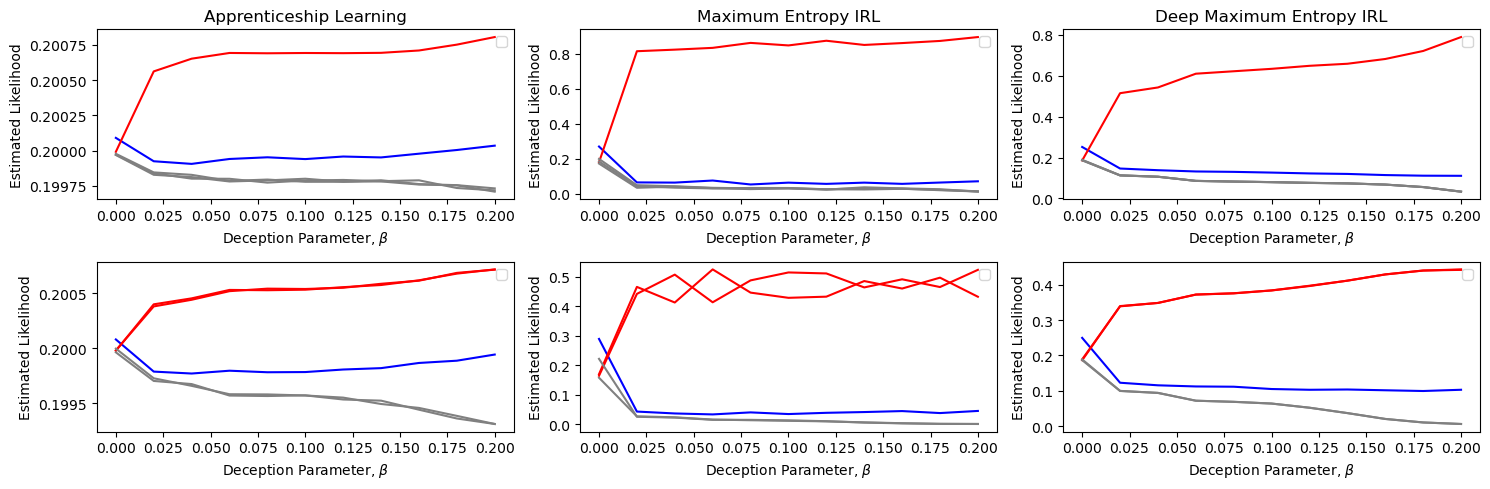}
    \caption{The adversary's estimate of~$L^i$ of each agent 
    under targeted deception with three types of IRL. Each plot shows results averaged over~$100$ runs.
    We used $\textit{number of trajectories} = 500$ and $\textit{length of trajectories} = 500$ for all simulations, $\textit{epoch} = 20$, $\textit{learning rate} = 0.5$ for MaxEnt IRL, and $\textit{epoch} = 50$, $\textit{learning rate} = 0.1$ and two fully connected hidden layers with dimension $64$ and $32$ for Deep IRL.
    The blue line represents the real agent, the red line(s) represents the target decoy agent(s), and the gray lines represent decoy agents.
    When~$\beta = 0$, no deception is implemented, and the adversary correctly
    infers that the real agent is the most important. 
    With~$\beta\geq 0.02$ the likelihood of the target decoy agent(s) is greater than that of the real agent, implying that deception misleads the adversary into believing that the target decoy agent(s) is the real one. Thus, the policy satisfies the definition of targeted deception.}
    \label{fig:targeted}
\end{figure}
\begin{table}
    \centering
    \begin{tabular}{c|c|c|cccccccccc}
    \hline \hline
         \multicolumn{2}{|c|}{Deception Parameter $\beta$} &0&0.02&0.04&0.06&0.08&0.1&0.12&0.14&0.16&0.18&0.2 \\
         \hline
         \multirow{6}{4.5em}{Scenario 1} 
         & Real &99 & 0& 0& 0& 0&  0&  0&  0&  0& 0&0\\
         & Target decoy &0 &100&100 &100& 100&  100&  100&  100&  100& 100&100\\
         & Decoy 1&0 & 0& 0& 0& 0&  0&  0&  0&  0& 0&0\\
         & Decoy 2&1& 0& 0& 0& 0&  0&  0&  0&  0& 0&0\\
         & Decoy 3&0 & 0& 0& 0& 0&  0&  0&  0&  0& 0&0\\
         \cline{2-13}
         & \textbf{Adversary's Utility}&\textbf{99}&\textbf{-100}& \textbf{-100}& \textbf{-100}& \textbf{-100}& \textbf{-100}& \textbf{-100}& \textbf{-100}& \textbf{-100}& \textbf{-100}&\textbf{-100}\\
         \hline
         \multirow{6}{4.5em}{Scenario 2} 
         & Real &100 & 0& 0& 0& 0&  0&  0&  0&  0& 0&0\\
         & Target decoy 1 &0&47& 50&52& 54&  54&  58&  55&  53& 46&56\\
         & Target decoy 2&0&53 &50&48& 46&  46&  42&  45&  47& 54&44\\
         & Decoy 1&0& 0& 0& 0& 0&  0&  0&  0&  0& 0&0\\
         & Decoy 2&0 & 0& 0& 0& 0&  0&  0&  0&  0& 0&0\\
         \cline{2-13}
         & \textbf{Adversary's Utility}& \textbf{100}& \textbf{-100}& \textbf{-100}& \textbf{-100}& \textbf{-100}& \textbf{-100}& \textbf{-100}& \textbf{-100}& \textbf{-100}& \textbf{-100}&\textbf{-100}\\
         \hline\hline
    \end{tabular}
    \caption{Adversary's estimation of the real agent and total utility when using Deep IRL under targeted deception}
    \label{tab:targeted}
\end{table}

Figure~\ref{fig:targeted} shows that targeted deception is attained for larger values of~$\beta$. 
With~$\beta\geq0.02$, the target decoy agent(s) has a larger~$L^i$ than the real agent, under all three types of IRL. This indicates that the adversary incorrectly infers that the target decoy agent(s) is the real agent.
Table~\ref{tab:targeted} shows that the adversary's estimation of the real agent is always one of the target decoy agents with $\beta\geq0.02$, resulting in a substantial decrease in the adversary's utility.
This implementation provides targeted deception by ensuring that the adversary draws a specific incorrect conclusion about the identity of the real agent.

\subsubsection{Equivocal Deception} 
From Definition~\ref{def:equivocal_deception}, a policy that is equivocal deceptive seeks to lead an adversary to a belief that two states are equally likely to be goal states. In the context of MTD, we seek to lead an adversary to infer that the real agent(s) and the target decoy agent(s) are equally important.
To verify this goal, we consider 2 scenarios with different ratios of real and decoy agents, where we have one real and one target decoy agents in scenario 1, and one real and two target decoy agents in scenario 2.
To achieve equivocal deception, we aim to  manipulate the occupancy measure $\{x(s,a)\}_{s\in\mathcal{S},a\in \mathcal{A}}$ to influence the metric~$L^i$ so that~$L^i$ of the real and target decoy agents have similar values.

\begin{figure}
    \centering
    \includegraphics[width=1\linewidth]{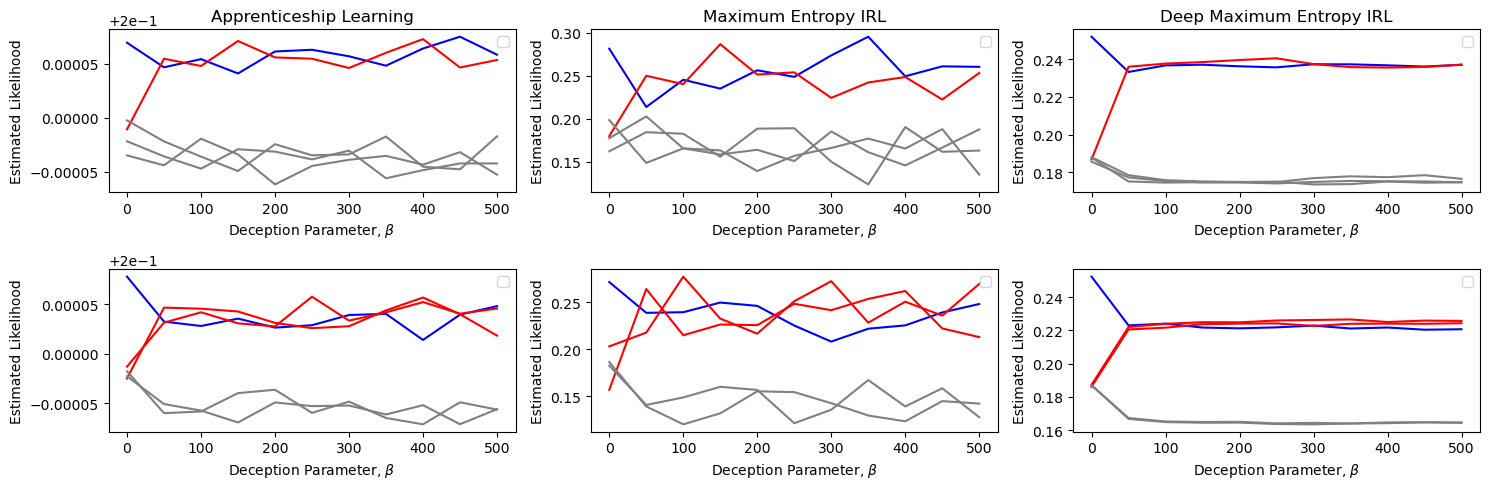}
    \caption{The adversary's estimate of~$L^i$ of each agent 
    under equivocal deception with three types of IRL. Each plot shows results averaged over~$100$ runs. 
    We used $\textit{number of trajectories} = 500$ and $\textit{length of trajectories} = 500$ for all simulations, $\textit{epoch} = 20$, $\textit{learning rate} = 0.5$ for MaxEnt IRL, and $\textit{epoch} = 50$, $\textit{learning rate} = 0.1$ and two fully connected hidden layers with dimension $64$ and $32$ for Deep IRL.
    The blue line represents the real agent, the red line(s) represents the target decoy agent(s), and the gray lines represent decoy agents.
    When~$\beta = 0$, no deception is implemented, and the adversary correctly
    infers that the real agent is the most important. 
    With~$\beta\geq 50$ the likelihoods of the real agent and the target decoy agent(s) remain close to each other, which satisfies the definition of equivocal deception.}
    \label{fig:equivocal}
\end{figure}
\begin{table}
    \centering
    \begin{tabular}{c|c|c|cccccccccc}
    \hline \hline
         \multicolumn{2}{|c|}{Deception Parameter $\beta$} &0&50&100&150&200&250&300&350&400&450&500 \\
         \hline
         \multirow{6}{4.5em}{Scenario 1} 
         & Real&99&36& 44& 36& 26& 19& 50& 56& 58& 45& 45\\
         & Target decoy&0 &64& 56& 64& 74& 81& 50& 44& 42& 55& 55\\
         & Decoy 1&0 & 0& 0& 0& 0&  0&  0&  0&  0& 0&0\\
         & Decoy 2&0& 0& 0& 0& 0&  0&  0&  0&  0& 0&0\\
         & Decoy 3&1 & 0& 0& 0& 0&  0&  0&  0&  0& 0&0\\
         \cline{2-13}
         & \textbf{Adversary's Utility}&\textbf{99}& \textbf{-28}& \textbf{-12}& \textbf{-28}& \textbf{-48}& \textbf{-62}&   \textbf{0}&  \textbf{12}&  \textbf{16}& \textbf{-10}& \textbf{-10}\\
         \hline
         \multirow{6}{4.5em}{Scenario 2} 
         & Real &100&  34&  38&  18&  10&  13&  27&  12&  15&   6&   6\\
         & Target decoy 1&0& 26&  25&  41&  47&  36&  19&  28&  41&  33&  32\\
         & Target decoy 2&0&40&  37&  41&  43&  51&  54&  60&  44&  61&  62\\
         & Decoy 1&0& 0& 0& 0& 0&  0&  0&  0&  0& 0&0\\
         & Decoy 2&0 & 0& 0& 0& 0&  0&  0&  0&  0& 0&0\\
         \cline{2-13}
         & \textbf{Adversary's Utility}&\textbf{100}& \textbf{-32}& \textbf{-24}& \textbf{-64}& \textbf{-80}& \textbf{-74}& \textbf{-46}& \textbf{-76}& \textbf{-70}& \textbf{-88}& \textbf{-88}\\
         \hline\hline
    \end{tabular}
    \caption{Adversary's estimation of the real agent and total utility when using Deep IRL under equivocal deception}
    \label{tab:equivocal}
\end{table}

Figure~\ref{fig:equivocal} shows that equivocal deception is achieved with larger values of~$\beta$. While Apprenticeship learning and MaxEnt IRL shows some fluctuation in their estimation, in general, with~$\beta\geq50$, the likelihoods of the real agent and all of the target decoy agent(s) are closely aligned, while the other decoy agents remain unimportant.
Table~\ref{tab:equivocal} shows that the adversary attacks the real agent less often with $\beta\geq 50$ compared to the case with no deception, resulting in a decrease in the adversary's utility. 
This setup provides equivocal deception by 
deceiving an adversary into believing there is no difference
in importance between the real and the target decoy agents. 


\subsection{Performance}
In this section, we evaluate the performance of the system when deception is implemented with a policy $\pi_d$ obtained from Optimization Problem~\ref{op:diversionary_optimization_problem},~\ref{op:targeted_optimization_problem}, or~\ref{op:modified_equivocal_optimization_problem}.
We quantify the revenue lost by introducing each type of deception and verify the bounds on the revenue loss given in Theorems~\ref{thm:bound_div},~\ref{thm:bound_tar}, and~\ref{thm:bound_equ}.

As shown in Figure~\ref{fig:performance}, all three types of deception can be achieved with only small losses in revenue.
In diversionary deception, with the choice of~$\beta=0.4$, we still recover~$98.04\%$ and~$99.07\%$ of the optimal revenue, which is less than a~{$2\%$} reduction compared to the revenue without deception.
Similarly, in targeted deception with~$\beta=0.02$ we find that the deceptive policy
recovers~{$99.95\%$} and~$99.90\%$ of the optimal revenue, and thus the loss in revenue in this case is less than~{$1\%$} compared to the revenue without deception.
In equivocal deception, with the choice of~{$\beta=50$} we observe that the deceptive policy attains~{$99.99\%$} and~$99.98\%$ of the optimal revenue, corresponding to a revenue loss of less than~{$1\%$} compared to the revenue without deception. Under each form of deception, we find that at values of~$\beta$ that generate deceptive policies, the resulting policies incur minimal losses in revenue, highlighting that we can achieve 
both deception and high performance simultaneously. 

\begin{figure}
    \centering
    \includegraphics[width=1\linewidth]{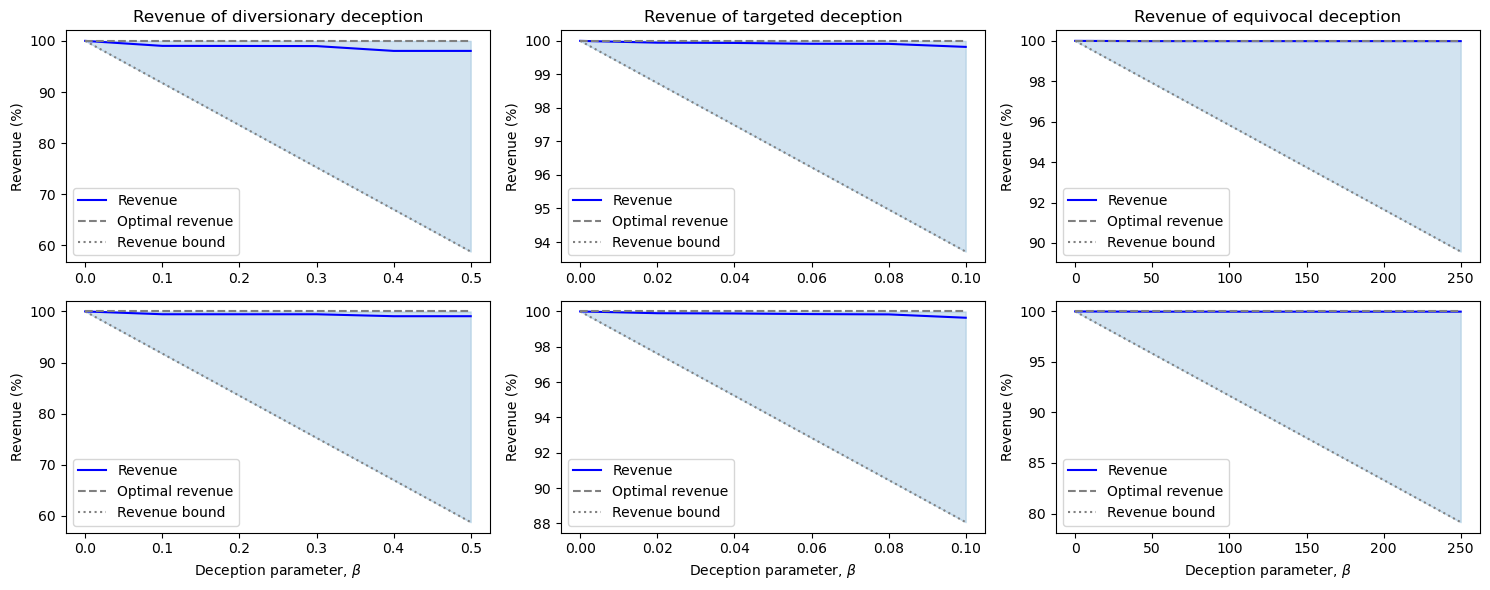}
    \caption{The revenue obtained from deceptive policies under diversionary, targeted, and equivocal deception. In diversionary deception, the system recovers~$98.04\%$ and~$99.07\%$ of the optimal revenue with $\beta = 0.4$. For targeted deception with~$\beta = 0.02$, the system recovers~{$99.95\%$} and~$99.90\%$ of the optimal revenue. For equivocal deception with~$\beta = 50$, 
    the system recovers~{$99.99\%$} and~$99.98\%$ of the optimal revenue. In all types of deception, less than $2\%$ of revenue is lost, 
    which outperforms worst-case bounds on performance loss, verifying Theorem~\ref{thm:bound_div},~\ref{thm:bound_tar} and~\ref{thm:bound_equ}.
    }
    \label{fig:performance}
\end{figure}

For practical implementation, Theorems~\ref{thm:bound_div},~\ref{thm:bound_tar} and~\ref{thm:bound_equ} can provide a guide for selecting parameter $\beta$. For some scenarios where we have a maximum allowable performance loss, we can find a range of~$\beta$ that attains required performance leveraging Theorems~\ref{thm:bound_div},~\ref{thm:bound_tar} and~\ref{thm:bound_equ}. In this example, the performance loss in scenario 2 of each types of deception are bounded by 
\begin{equation}
    L_{\pi_{div}} \leq 0.82432\frac{\beta}{R^*}, \quad
    L_{\pi_{tar}} \leq 0.62866\frac{\beta}{R^*}, \quad
    L_{\pi_{equ}} \leq 0.00042\frac{\beta}{R^*}.
\end{equation}
For example, to achieve $80\%$ of it's optimal performance, we need $0\leq \beta \lesssim  0.243$ for diversionary deception, $0\leq \beta \lesssim 0.168$ for targeted deception, and $0\leq \beta \lesssim 239.7$ for equivocal deception.
This indicates that the choice of~$\beta \simeq0.243$,~$\beta\simeq0.168$ and~$\beta\simeq239.7$ can provide the most deceptive policy under allowable performance loss.

From Table~\ref{tab:diversionary},~\ref{tab:targeted}, and~\ref{tab:equivocal}, we find $\beta=0.2$ in diversionary deception (scenario 2) drops the adversary's utility from $100$ to $34$, $\beta = 0.16$ in targeted deception (scenario 2) drops it to $-100$, and $\beta=200$ in equivocal deception (scenario 2) drops it to $-80$, while $80\%$ of its optimal performance is theoretically guaranteed in all cases. However, in practice, Figure~\ref{fig:performance} shows that the actual performance outperforms these bounds on performance loss in all three types of deception.

\section{Conclusion} \label{sec:conclude}

This study demonstrates that active deception strategies can effectively mislead adversaries in fully observable systems with formal performance guarantees under a specified degree of deception.
We have developed analytic bounds for performance losses in terms of total accumulated rewards and verified the performance 
of deceptive decision policies with various scenarios
in simulations.
By ensuring reliable operation under deception, our findings enhance the practical applicability of deception in securing critical systems.
Future work will explore 
tradeoffs between optimality and 
deception under a broader range of adversarial models, including some model that can expect and react to deception, and alternate definitions of deception.

\section*{Acknowledgments}
This work was supported by Defense Advanced Research Projects Agency under Grant No. HR00112420348, Office of Naval Research under Grant No. N00014-24-1-2432, National Science Foundation Graduate Research Fellowship under Grant No. DGE-2039655. Any opinions, findings and conclusions or recommendations expressed herein are those of the authors and do not necessarily reflect the views of sponsoring agencies.

\section*{Disclosure Statement}
The authors report there are no competing interests to declare.

\section*{Data Availability Statement}
The data that support the findings of this study were generated by the authors and are reproducible using the code publicly available at https://anonymous.4open.science/r/DeceptionMTD-2980. No external datasets were used.

%
%
%
\begin{appendices}

\section{Proposition~\ref{prop:oc_sum}}\label{prop_oc}

We state the following proposition and use it in the proofs of Theorems~\ref{thm:bound_div} and~\ref{thm:bound_tar}.

\begin{proposition}\label{prop:oc_sum}
The sum of the values of the occupancy measure~$x$ across all states and actions is 
\begin{equation}\label{eq:oc_sum}
    \sum_{s \in \mathcal{S}} \sum_{a\in \mathcal{A}}x(s,a) = (1-\gamma)^{-1}.
\end{equation}
Moreover, the sum of the squares
of the occupancy measure~$x$ across all states and actions is bounded via 
\begin{equation}\label{eq:oc_square_sum}
    \frac{(1-\gamma)^{-2}}{|\mathcal{S}||\mathcal{A}|} \leq \sum_{s \in \mathcal{S}} \sum_{a\in \mathcal{A}}x(s,a)^2 \leq (1-\gamma)^{-2}.
\end{equation} 
\end{proposition}

\begin{proof}
Denote the feasible region of Optimization Problem~\ref{op:dualLP_with_task_constraint} by $\mathcal{X}$. Then all choices of $x \in \mathcal{X}$ satisfy the constraint
\[\sum_{a\in \mathcal{A}} x(j,a) - \gamma \sum_{s\in \mathcal{S}}\sum_{a \in \mathcal{A}} \mathcal{T}(s,a,j)x(s,a) = \alpha(j)\]
for all~$j \in \mathcal{S}$. 
By summing this equation over all $j \in \mathcal{S}$, we find 
\begin{equation}\label{eq:sum_flow_constraint}
\begin{aligned}
    \sum_{j \in \mathcal{S}} \sum_{a\in \mathcal{A}}x(j,a) - \gamma \sum_{j \in \mathcal{S}}\sum_{s\in \mathcal{S}}\sum_{a \in \mathcal{A}} \mathcal{T}(s,a,j)x(s,a) = 1.
\end{aligned}
\end{equation}
By substituting the definition of an occupancy measure from \eqref{eq:def_occupancy_measure} into the second term, we obtain
\begin{equation}
\begin{aligned}
        & \gamma \sum_{j \in \mathcal{S}}\sum_{s\in \mathcal{S}}\sum_{a \in \mathcal{A}} \mathcal{T}(s,a,j)x(s,a) \\
         = &\gamma  \sum_{j \in \mathcal{S}}\sum_{s\in \mathcal{S}}\sum_{a \in \mathcal{A}} \mathcal{T}(s,a,j) \sum_{k \in \mathcal{S}} \alpha(k) \sum_{n=1}^\infty \gamma^{n-1} P^\policy_{n,k}(s,a)\\
         = &\gamma \sum_{k \in \mathcal{S}} \alpha(k) \sum_{n=1}^\infty \gamma^{n-1}  \sum_{j \in \mathcal{S}}\sum_{s\in \mathcal{S}}\sum_{a \in \mathcal{A}} \mathcal{T}(s,a,j)P^\policy_{n,k}(s,a)
\end{aligned}
\end{equation}
for the policy $\policy$ generated by $x$, where the last equality follows from reordering the summation.
The sum of the probability $P^\policy_{n,k}(s,a)$ multiplied by the transition probability $\mathcal{T}(s,a,j)$ over all $s\in \mathcal{S}$ and $a\in\mathcal{A}$ gives the probability of being in state $j$ at timestep $n+1$, that is,
\begin{equation}
\begin{aligned}
& \gamma \sum_{k \in \mathcal{S}} \alpha(k) \sum_{n=1}^\infty \gamma^{n-1}  \sum_{j \in \mathcal{S}}\sum_{s\in \mathcal{S}}\sum_{a \in \mathcal{A}} \mathcal{T}(s,a,j)P^\policy_{n,k}(s,a)\\
& = \gamma \sum_{k \in \mathcal{S}} \alpha(k) \sum_{n=1}^\infty \gamma^{n-1} \sum_{j \in \mathcal{S}}\sum_{a\in \mathcal{A}}P^\policy_{n+1,k}(j,a). 
\end{aligned}
\end{equation}
We know that $\sum_{j \in \mathcal{S}}\sum_{a\in \mathcal{A}}P^\policy_{n+1,k}(j,a) = 1$ because this quantity is the probability of being in any state and taking any action at timestep $n+1$. Therefore
\begin{equation}\label{eq:plug_back_into}
\begin{aligned}
& \gamma \sum_{k \in \mathcal{S}} \alpha(k) \sum_{n=1}^\infty \gamma^{n-1} \sum_{j \in \mathcal{S}}\sum_{a\in \mathcal{A}}P^\policy_{n+1,k}(j,a) = \gamma \sum_{k \in \mathcal{S}} \alpha(k) \sum_{n=1}^\infty \gamma^{n-1}
    =\gamma(1-\gamma)^{-1},
\end{aligned}
\end{equation}
where the last equality follows from $\sum_{k\in\mathcal{S}}\idist(k) = 1$ and 
the fact that~$\gamma\in[0,1)$.
Substituting~\eqref{eq:plug_back_into} in to~\eqref{eq:sum_flow_constraint} gives
\begin{equation}
    \sum_{j \in \mathcal{S}} \sum_{a\in \mathcal{A}}x(j,a) = 1+ \gamma(1-\gamma)^{-1} = (1-\gamma)^{-1}.
\end{equation}

Next define $X \in \mathbb{R}^{|\mathcal{S}||\mathcal{A}|}$ as the vectorized form of the set of occupancy measures $\{x(s,a)\}_{s\in\mathcal{S}, a\in\mathcal{A}}$.
Since $x(s,a) \geq 0$ for all $s \in \mathcal{S}$ and $a\in \mathcal{A}$,
we find that 
\begin{equation}
    \norm{X}_1 = \sum_{s\in \mathcal{S}}\sum_{a\in \mathcal{A}} |x(s,a)| = (1-\gamma)^{-1}.
\end{equation}
From the standard bounds relating the~$1$-norm and the~$2$-norm, we have 
 $\frac{1}{\sqrt{|\mathcal{S}||\mathcal{A}|}}\norm{X}_1 \leq \norm{X}_2 \leq \norm{X}_1$, from which we find 
\begin{equation}
    \frac{(1-\gamma)^{-2}}{|\mathcal{S}||\mathcal{A}|} \leq \sum_{s \in \mathcal{S}} \sum_{a\in \mathcal{A}}x(s,a)^2 \leq (1-\gamma)^{-2},
\end{equation}
as desired.
\end{proof}

\section{Proof of Theorem~\ref{thm:bound_div}}\label{prf:bound_div}
Denote the feasible region of Optimization Problem~\ref{op:diversionary_optimization_problem} by $\mathcal{X}$. 
Then Optimization Problem~\ref{op:diversionary_optimization_problem} 
can be written as 
\begin{equation}
\begin{aligned}
\underset{ x\in \mathcal{X}}{\operatorname{maximize}} &\sum_{s\in \mathcal{S}}\sum_{a\in \mathcal{A}}\Big[r(s,a)x(s,a) + \beta\big(x(s,a) - x^*(s,a)\big)^2\Big],
\end{aligned}
\end{equation}
where $x^*$ is the solution to Optimization Problem~\ref{op:dualLP_with_task_constraint}, which does not include additional terms for deception.
Since $r(s,a)$ and $x^*(s,a)$ are fixed values for all $s \in \mathcal{S}$ and $a \in \mathcal{A}$, the solution to the problem remains identical after subtracting the constant term $\sum_{s\in \mathcal{S}}\sum_{a\in \mathcal{A}} r(s,a)x^*(s,a)$ from the objective function.
That is, we may equivalently solve the problem 
\begin{equation} \begin{aligned}
    \underset{ x\in \mathcal{X}}{\operatorname{maximize}} &\sum_{s\in \mathcal{S}}\sum_{a\in \mathcal{A}} \Big[r(s,a)\big(x(s,a)-x^*(s,a)\big) + \beta\big(x(s,a) - x^*(s,a)\big)^2\Big].
\end{aligned} \end{equation}

We observe that 
$x^* \in \mathcal{X}$ and substituting $x = x^*$ makes the objective function $0$. 
Let $x_d$ be a solution of Optimization Problem~\ref{op:diversionary_optimization_problem}. Then $x_d$ must achieve a 
value of the objective function that is no smaller than that
attained by any other feasible solution, including $x^*$. Then
\begin{equation} \begin{aligned}
    \sum_{s\in \mathcal{S}}\sum_{a\in \mathcal{A}} \Big[ r(s,a)(x_d(s,a) - x^*(s,a))  \beta\big(x_d(s,a)-x^*(s,a)\big)^2 \Big] \geq 0.
\end{aligned} \end{equation}
Let $\policy_d$ be the policy obtained from $x_d$ by~\eqref{eq:oc_policy}.
Then, rearranging this inequality and applying~\eqref{eq:revenue} gives 
\begin{equation} \begin{aligned}
    R^* - R_{\policy_d}
    &= \sum_{s\in \mathcal{S}}\sum_{a\in \mathcal{A}} r(s,a)(x^*(s,a) - x_d(s,a)) \\
    &\leq \beta \sum_{s\in \mathcal{S}}\sum_{a\in \mathcal{A}} (x^*(s,a) - x_d(s,a))^2\\
    & = \beta \sum_{s\in \mathcal{S}}\sum_{a\in \mathcal{A}} \Big(x^*(s,a)^2 - 2x^*(s,a)x_d(s,a) + x_d(s,a)^2\Big)\\
    & \leq \beta \left(\sum_{s\in \mathcal{S}}\sum_{a\in \mathcal{A}} x^*(s,a)^2  + \sum_{s\in \mathcal{S}}\sum_{a\in \mathcal{A}}x_d(s,a)^2\right),
\end{aligned} \end{equation}
where the last inequality follows from the non-negativity of occupancy measures.
Since we know that $\sum_{s\in \mathcal{S}}\sum_{a\in \mathcal{A}}x_d(s,a)^2 \leq (1-\gamma)^{-2}$ from Proposition~\ref{prop:oc_sum} and since $x^*$ is fixed, we find the bound of the revenue loss when employing the diversionary deceptive policy $\policy_d$ as
\begin{equation} \label{eq:div_loss_bound}
    L_{\policy_d}\leq \frac{\beta}{R^*}\left(\sum_{s \in\mathcal{S}}\sum_{a\in\mathcal{A}} x^*(s,a)^2 + (1-\gamma)^{-2}\right),
\end{equation}
where $L_{\policy_d}$ is from~\eqref{eq:revenue_loss}.

\section{Proof of Theorem~\ref{thm:bound_tar}}\label{prf:bound_tar}
Denote the feasible region of Optimization Problem~\ref{op:targeted_optimization_problem} by $\mathcal{X}$. Then Optimization Problem~\ref{op:targeted_optimization_problem} becomes
\begin{equation} \begin{aligned}
    \underset{x\in \mathcal{X}}{\operatorname{maximize}} \sum_{s\in \mathcal{S}}\sum_{a\in \mathcal{A}} \Big[r(s,a)x(s,a)  - \beta\big(x(s,a) - x_{tar}(s,a)\big)^2\Big],
\end{aligned} \end{equation}
where $x_{tar}$ is a set of user-specified target occupancy measures that reflects the specific incorrect goal states we aim to 
make the adversary infer.
As with diversionary deception,~$r(s,a)$ and~$x^*(s,a)$ are fixed values for all $s \in \mathcal{S}$ and $a \in \mathcal{A}$, and the solution to the problem remains identical after subtracting the term $\sum_{s\in \mathcal{S}}\sum_{a\in \mathcal{A}} r(s,a)x^*(s,a)$ from the objective function.
That is, we may equivalently write the problem as:
\begin{equation} \begin{aligned}
    \underset{x\in \mathcal{X}}{\operatorname{maximize}} \sum_{s\in \mathcal{S}}\sum_{a\in \mathcal{A}} \Big[r(s,a)(x(s,a)-x^*(s,a))  - \beta\big(x(s,a) - x_{tar}(s,a)\big)^2\Big].
\end{aligned} \end{equation}
Recall that $x^*$ is the solution to Optimization Problem~\ref{op:dualLP_with_task_constraint}, which does not include additional terms for deception.
Notice as well that~$x^* \in \mathcal{X}$ and substituting in~$x = x^*$ makes the objective function 
take the value
$-\sum_{s\in \mathcal{S}}\sum_{a\in \mathcal{A}}\beta(x^*(s,a)-x_{tar}(s,a))^2$. 
Let~$x_d$ be a solution of Optimization Problem~\ref{op:targeted_optimization_problem}. Then~$x_d$ achieves a value 
of the objective function 
that is no smaller
than that obtained by 
any other feasible solution, including $x^*$. Then
\begin{equation} \begin{aligned}
    \sum_{s\in \mathcal{S}}\sum_{a\in \mathcal{A}} \Big[r(s,a)(x_d(s,a) - x^*(s,a)) - \beta(x_d(s,a)-x_{tar}(s,a))^2\Big] \\ \geq -\sum_{s\in \mathcal{S}}\sum_{a\in \mathcal{A}} \beta(x^*(s,a)-x_{tar}(s,a))^2.
\end{aligned} \end{equation}
Let $\policy_d$ be the policy obtained from $x_d$ by~\eqref{eq:oc_policy}.
Then by rearranging terms and applying~\eqref{eq:revenue}, we find
\begin{equation} \begin{aligned}
     R^* - R_{\policy_d}  & = \sum_{s\in \mathcal{S}}\sum_{a\in \mathcal{A}} r(s,a)(x^*(s,a) - x_d(s,a)) \\
    & \leq  \beta \sum_{s\in \mathcal{S}}\sum_{a\in \mathcal{A}}\big[\big(x^*(s,a)-x_{tar}(s,a)\big)^2 - \big(x_d(s,a)-x_{tar}(s,a)\big)^2 \big],
\end{aligned} \end{equation}
and expanding gives
\begin{equation} \begin{aligned}
    &R^*- R_{\policy_d} \leq \beta \sum_{s\in \mathcal{S}}\sum_{a\in \mathcal{A}}
    \Big[(x^*(s,a)^2-2x^*(s,a)x_{tar}(s,a) - x_d(s,a)^2 +2x_d(s,a)x_{tar}(s,a)\Big].
\end{aligned} \end{equation}
We know that $\frac{(1-\gamma)^{-2}}{|\mathcal{S}||\mathcal{A}|} \leq \sum_{s \in \mathcal{S}} \sum_{a\in \mathcal{A}}x_d(s,a)^2$
and $\sum_{s\in\mathcal{S}}\sum_{a\in\mathcal{A}}x_d(s,a) \leq (1-\gamma)^{-1}$ from Proposition~\ref{prop:oc_sum}.
Since $x^*$ and $x_{tar}$ are fixed, we bound the revenue loss when employing a targeted deceptive policy $\policy_d$ as

\begin{equation}\label{eq:tar_loss_bound}
\begin{aligned}
    &L_{\policy_d} \leq \frac{\beta}{R^*}\left[ \sum_{s \in\mathcal{S}}\sum_{a\in\mathcal{A}} \Big(x^*(s,a)^2 - 2x_{tar}(s,a)x^*(s,a)\Big)-\frac{(1-\gamma)^{-2}}{|\mathcal{S}||\mathcal{A}|}  +2(1-\gamma)^{-1}\max_{s\in\mathcal{S},a\in\mathcal{A}}x_{tar}(s,a)\right],
\end{aligned} 
\end{equation}
where $L_{\policy_d}$ is from~\eqref{eq:revenue_loss}.

\section{Proof of Theorem~\ref{thm:bound_equ}}\label{prf:bound_equ}

Denote the feasible region of Optimization Problem~\ref{op:modified_equivocal_optimization_problem} by $\mathcal{X}$. 
Then, Optimization Problem~\ref{op:modified_equivocal_optimization_problem} becomes
\begin{equation} \begin{aligned}
    & \underset{x\in \mathcal{X}}{\operatorname{maximize}} \sum_{s\in \mathcal{S}}\sum_{a\in \mathcal{A}} r(s,a)x(s,a) - \beta \Big(\sum_{s\in\mathcal{S}_{goal}}\sum_{a\in\mathcal{A}}x(s, a) -\sum_{s\in\mathcal{S}_{decoy}}\sum_{a\in\mathcal{A}}x(s, a)\Big)^2.
\end{aligned} \end{equation}
Let $x^*$ be the solution to Optimization Problem~\ref{op:dualLP_with_task_constraint}, which does not include additional terms for deception. 
Since $r(s,a)$ and~$x^*(s,a)$ are fixed values for all $s \in \mathcal{S}$ and $a \in \mathcal{A}$ and the solution to the problem remains identical after subtracting the term $\sum_{s\in \mathcal{S}}\sum_{a\in \mathcal{A}} r(s,a)x^*(s,a)$ from the objective function, we may equivalently write the problem as:
\begin{equation} \begin{aligned}
    & \underset{x\in \mathcal{X}}{\operatorname{maximize}}\sum_{s\in \mathcal{S}}\sum_{a\in \mathcal{A}} r(s,a)(x(s,a)-x^*(s,a)) - \beta \Big(\sum_{s\in\mathcal{S}_{goal}}\sum_{a\in\mathcal{A}}x(s, a) -\sum_{s\in\mathcal{S}_{decoy}}\sum_{a\in\mathcal{A}}x(s, a)\Big)^2.
\end{aligned}  \end{equation}
Notice as well that~$x^* \in \mathcal{X}$, and substituting in~$x = x^*$ yields the objective function value
$-\beta (\sum_{s\in\mathcal{S}_{goal}}\sum_{a\in\mathcal{A}}x^*(s, a) -\sum_{s\in\mathcal{S}_{decoy}}\sum_{a\in\mathcal{A}}x^*(s, a))^2$. 
The solution~$x_d$ achieves a value 
of the objective function 
that is no smaller
than that obtained by 
any other feasible solution, including $x^*$. Then 
\begin{equation} \begin{aligned}
     \sum_{s\in \mathcal{S}}\sum_{a\in \mathcal{A}} r(s,a)(x_d(s,a)-x^*(s,a))- \beta (\sum_{s\in\mathcal{S}_{goal}}\sum_{a\in\mathcal{A}}x_d(s, a) -\sum_{s\in\mathcal{S}_{decoy}}\sum_{a\in\mathcal{A}}x_d(s, a))^2 \\ 
    \geq -\beta (\sum_{s\in\mathcal{S}_{goal}}\sum_{a\in\mathcal{A}}x^*(s, a) -\sum_{s\in\mathcal{S}_{decoy}}\sum_{a\in\mathcal{A}}x^*(s, a))^2.
\end{aligned} \end{equation}
Let $\policy_d$ be the policy obtained from $x_d$ by~\eqref{eq:oc_policy}.
Then by rearranging terms and applying~\eqref{eq:revenue}, we find
\begin{equation} \begin{aligned}
    & R^* - R_{\policy_d} = \sum_{s\in \mathcal{S}}\sum_{a\in \mathcal{A}} r(s,a)(x^*(s,a) - x_d(s,a)) \\
    & \leq \beta (\sum_{s\in\mathcal{S}_{goal}}\sum_{a\in\mathcal{A}}x^*(s, a) - \sum_{s\in\mathcal{S}_{decoy}}\sum_{a\in\mathcal{A}}x^*(s, a))^2 \\ &  - \beta (\sum_{s\in\mathcal{S}_{goal}}\sum_{a\in\mathcal{A}}x_d(s, a) -\sum_{s\in\mathcal{S}_{decoy}}\sum_{a\in\mathcal{A}}x_d(s, a))^2 .
\end{aligned} \end{equation}
We have $\beta \big(\sum_{s\in\mathcal{S}_{goal}}\sum_{a\in\mathcal{A}}x_d(s, a) -\sum_{s\in\mathcal{S}_{decoy}}\sum_{a\in\mathcal{A}}x_d(s, a)\big)^2  \geq 0$. 
Using this fact and the fact that $x^*$ is fixed, we bound the revenue loss when employing a targeted deceptive policy $\policy_d$ as
\begin{equation}\label{eq:equ_loss_bound}
\begin{aligned}
    & L_{\policy_d} \leq  \frac{\beta}{R^*}\left(\sum_{s\in\mathcal{S}_{goal}}\sum_{a\in\mathcal{A}}x^*(s, a) - \sum_{s\in\mathcal{S}_{decoy}}\sum_{a\in\mathcal{A}}x^*(s, a)\right)^2,
\end{aligned} 
\end{equation}
where $L_{\policy_d}$ is from~\eqref{eq:revenue_loss}.

\end{appendices}

\bibliographystyle{IEEEtran}
\bibliography{JORS/main} 

@article{qayyum2020securing,
  title={Securing connected \& autonomous vehicles: Challenges posed by adversarial machine learning and the way forward},
  author={Qayyum, Adnan and Usama, Muhammad and Qadir, Junaid and Al-Fuqaha, Ala},
  journal={IEEE Communications Surveys \& Tutorials},
  volume={22},
  number={2},
  pages={998--1026},
  year={2020},
  publisher={IEEE}
}

@book{puterman2014markov,
  title={Markov decision processes: discrete stochastic dynamic programming},
  author={Puterman, Martin L},
  year={2014},
  publisher={John Wiley \& Sons}
}

@inproceedings{boutilier1996planning,
  title={Planning, learning and coordination in multiagent decision processes},
  author={Boutilier, Craig},
  booktitle={Theoretical Aspects of Rationality and Knowledge (TARK)},
  volume={96},
  pages={195--210},
  year={1996},
  organization={Citeseer}
}

@inproceedings{ramachandran2007bayesian,
  title={Bayesian Inverse Reinforcement Learning.},
  author={Ramachandran, Deepak and Amir, Eyal},
  booktitle={International Joint Conferences on Artificial Intelligence (IJCAI)},
  volume={7},
  pages={2586--2591},
  year={2007}
}

@article{lv2024optimal,
  title={Optimal Deceptive Strategy Synthesis for Autonomous Systems under Asymmetric Information},
  author={Lv, Peng and Li, Shaoyuan and Yin, Xiang},
  journal={IEEE Transactions on Intelligent Vehicles},
  year={2024},
  pages={6108-6121},
  publisher={IEEE}
}

@article{wulfmeier2015maximum,
  title={Maximum entropy deep inverse reinforcement learning},
  author={Wulfmeier, Markus and Ondruska, Peter and Posner, Ingmar},
  journal={arXiv preprint arXiv:1507.04888},
  year={2015}
}

@inproceedings{mceneaney2005deception,
  title={Deception in autonomous vehicle decision making in an adversarial environment},
  author={McEneaney, William and Singh, Rajdeep},
  booktitle={AIAA Guidance, Navigation, and Control Conference and Exhibit},
  pages={6152},
  year={2005}
}

@article{arkin1990autonomous,
  title={Autonomous navigation in a manufacturing environment},
  author={Arkin, Ronald C and Murphy, Robin R},
  journal={IEEE Transactions on Robotics and Automation},
  volume={6},
  number={4},
  pages={445--454},
  year={1990},
  publisher={IEEE}
}

@article{SHEEHAN2019523,
title = {Connected and autonomous vehicles: A cyber-risk classification framework},
journal = {Transportation Research Part A: Policy and Practice},
volume = {124},
pages = {523-536},
year = {2019},
author = {Barry Sheehan and Finbarr Murphy and Martin Mullins and Cian Ryan}
}

@INPROCEEDINGS{9519418,
  author={Cheu, Albert and Smith, Adam and Ullman, Jonathan},
  booktitle={2021 IEEE Symposium on Security and Privacy (SP)}, 
  title={Manipulation Attacks in Local Differential Privacy}, 
  year={2021},
  volume={},
  number={},
  pages={883-900}}

@INPROCEEDINGS{9304015,
  author={Gohari, Parham and Hale, Matthew and Topcu, Ufuk},
  booktitle={2020 59th IEEE Conference on Decision and Control (CDC)}, 
  title={Privacy-Preserving Policy Synthesis in Markov Decision Processes}, 
  year={2020},
  volume={},
  number={},
  pages={6266-6271}
}

@article{ying2020note,
  title={A note on optimization formulations of Markov decision processes},
  author={Ying, Lexing and Zhu, Yuhua},
  journal={arXiv preprint arXiv:2012.09417},
  year={2020}
}

@inproceedings{ziebart2008maximum,
  title={Maximum entropy inverse reinforcement learning.},
  author={Ziebart, Brian D and Maas, Andrew L and Bagnell, J Andrew and Dey, Anind K and others},
  booktitle={Association for the Advancement of Artificial Intelligence (AAAI)},
  volume={8},
  pages={1433--1438},
  year={2008}
}

@article{mhara2024cyber,
  title={Cyber attacks and threats: Study of the types of cyber attacks: Hacking, viruses, targeted attacks, and electronic espionage},
  author={Mhara, Mustafa AO Abo and Abdulrahman, Abdullah AA and Baroud, Abdulhakim AS},
  journal={ International Journal of Electrical Engineering and Sustainability},
  pages={38--47},
  year={2024}
}

@article{10.1145/989.991,
author = {Landwehr, Carl E. and Heitmeyer, Constance L. and McLean, John},
title = {A security model for military message systems},
year = {1984},
issue_date = {Aug. 1984},
publisher = {Association for Computing Machinery},
address = {New York, NY, USA},
volume = {2},
number = {3},
journal = {ACM Transactions on Computer Systems},
month = aug,
pages = {198–222},
numpages = {25}
}

@INPROCEEDINGS{9152764,
  author={Angel, Sebastian and Kannan, Sampath and Ratliff, Zachary},
  booktitle={2020 IEEE Symposium on Security and Privacy (SP)}, 
  title={Private resource allocators and their applications}, 
  year={2020},
  volume={},
  number={},
  pages={372-391}
}

@inproceedings{abbeel2004apprenticeship,
  title={Apprenticeship learning via inverse reinforcement learning},
  author={Abbeel, Pieter and Ng, Andrew Y},
  booktitle={Proceedings of the twenty-first International Conference on Machine learning},
  pages={1},
  year={2004}
}

@inproceedings{chen2017cyber,
  title={Cyber-physical system enabled nearby traffic flow modelling for autonomous vehicles},
  author={Chen, Baiyu and Yang, Zhengyu and Huang, Siyu and Du, Xianzhi and Cui, Zhiwei and Bhimani, Janki and Xie, Xin and Mi, Ningfang},
  booktitle={2017 IEEE 36th International Performance Computing and Communications Csonference (IPCCC)},
  pages={1--6},
  year={2017},
  organization={IEEE}
}

@article{monostori2016cyber,
  title={Cyber-physical systems in manufacturing},
  author={Monostori, L{\'a}szl{\'o} and K{\'a}d{\'a}r, Botond and Bauernhansl, Thomas and Kondoh, Shinsuke and Kumara, Soundar and Reinhart, Gunther and Sauer, Olaf and Schuh, Gunther and Sihn, Wilfried and Ueda, Kenichi},
  journal={CIRP Annals},
  volume={65},
  number={2},
  pages={621--641},
  year={2016},
  publisher={Elsevier}
}

@article{guo2022cyber,
  title={Cyber-physical system-based path tracking control of autonomous vehicles under cyber-attacks},
  author={Guo, Jinghua and Li, Lubin and Wang, Jingyao and Li, Keqiang},
  journal={IEEE Transactions on Industrial Informatics},
  volume={19},
  number={5},
  pages={6624--6635},
  year={2022},
  publisher={IEEE}
}

@article{yu2016smart,
  title={Smart grids: A cyber--physical systems perspective},
  author={Yu, Xinghuo and Xue, Yusheng},
  journal={Proceedings of the IEEE},
  volume={104},
  number={5},
  pages={1058--1070},
  year={2016},
  publisher={IEEE}
}

@article{yazdani2022differentially,
  title={Differentially private {LQ} control},
  author={Yazdani, Kasra and Jones, Austin and Leahy, Kevin and Hale, Matthew},
  journal={IEEE Transactions on Automatic Control},
  year={2022},
volume={68},
  number={2},
  pages={1061--1068},
  publisher={IEEE}
}

@inproceedings{hawkins2020differentially,
  title={Differentially private formation control},
  author={Hawkins, Calvin and Hale, Matthew},
  booktitle={59th IEEE Conference on Decision and Control (CDC)},
  pages={6260--6265},
  year={2020},
}

@inproceedings{chen2023differential,
  title={Differential privacy in cooperative multiagent planning},
  author={Chen, Bo and Hawkins, Calvin and Karabag, Mustafa O and Neary, Cyrus and Hale, Matthew and Topcu, Ufuk},
  booktitle={Uncertainty in Artificial Intelligence},
  pages={347--357},
  year={2023},
  organization={PMLR}
}

@article{chen2023differentialsymbolic,
  title={Differential privacy for symbolic systems with application to Markov Chains},
  author={Chen, Bo and Leahy, Kevin and Jones, Austin and Hale, Matthew},
  journal={Automatica},
  volume={152},
  pages={110908},
  year={2023},
  publisher={Elsevier}
}

@inproceedings{benvenuti2023differentially,
  title={Differentially Private Reward Functions for Markov Decision Processes},
  author={Benvenuti, Alexander and Hawkins, Calvin and Fallin, Brandon and Chen, Bo and Bialy, Brendan and Dennis, Miriam and Hale, Matthew},
  booktitle={2024 IEEE Conference on Control Technology and Applications (CCTA)},
  pages={631--636},
  year={2024},
  organization={IEEE}
}

@inproceedings{karabag2019least,
  title={Least inferable policies for Markov decision processes},
  author={Karabag, Mustafa O and Ornik, Melkior and Topcu, Ufuk},
  booktitle={2019 American Control Conference (ACC)},
  pages={1224--1231},
  year={2019},
  organization={IEEE}
}

@article{zheng2019markov,
  title={Markov decision process to enforce moving target defence policies},
  author={Zheng, Jianjun and Namin, Akbar Siami},
  journal={arXiv preprint arXiv:1905.09222},
  year={2019}
}

@INPROCEEDINGS{9833672,
  author={Jin, Jiankai and McMurtry, Eleanor and Rubinstein, Benjamin I. P. and Ohrimenko, Olga},
  booktitle={2022 IEEE Symposium on Security and Privacy (SP)}, 
  title={Are We There Yet? Timing and Floating-Point Attacks on Differential Privacy Systems}, 
  year={2022},
  volume={},
  number={},
  pages={473-488}}

@article{karabag2021deception,
  title={Deception in supervisory control},
  author={Karabag, Mustafa O and Ornik, Melkior and Topcu, Ufuk},
  journal={IEEE Transactions on Automatic Control},
  volume={67},
  number={2},
  pages={738--753},
  year={2021},
  publisher={IEEE}
}

@article{savas2019entropy,
  title={Entropy maximization for Markov decision processes under temporal logic constraints},
  author={Savas, Yagiz and Ornik, Melkior and Cubuktepe, Murat and Karabag, Mustafa O and Topcu, Ufuk},
  journal={IEEE Transactions on Automatic Control},
  volume={65},
  number={4},
  pages={1552--1567},
  year={2019},
  publisher={IEEE}
}

@article{karabag2022exploiting,
  title={Exploiting partial observability for optimal deception},
  author={Karabag, Mustafa O and Ornik, Melkior and Topcu, Ufuk},
  journal={IEEE Transactions on Automatic Control},
  volume={68},
  number={7},
  pages={4443--4450},
  year={2022},
  publisher={IEEE}
}

@inproceedings{abdulhai2024defining,
  title={Defining Deception in Decision Making},
  author={Abdulhai, Marwa and Carroll, Micah and Svegliato, Justin and Dragan, Anca and Levine, Sergey},
  booktitle={Proceedings of the 23rd International Conference on Autonomous Agents and Multiagent Systems},
  pages={2111--2113},
  year={2024}
}

@inproceedings{kim2024defining,
  title={Defining and Measuring Deception in Sequential Decision Systems: Application to Network Defense},
  author={Kim, Yerin and Benvenuti, Alexander and Chen, Bo and Karabag, Mustafa and Kulkarni, Abhishek and Bastian, Nathaniel D and Topcu, Ufuk and Hale, Matthew},
  booktitle={MILCOM 2024-2024 IEEE Military Communications Conference (MILCOM)},
  pages={1--6},
  year={2024},
  organization={IEEE}
}

@INPROCEEDINGS{6614155,
  author={Jia, Quan and Sun, Kun and Stavrou, Angelos},
  booktitle={2013 22nd International Conference on Computer Communication and Networks (ICCCN)}, 
  title={MOTAG: Moving Target Defense against Internet Denial of Service Attacks}, 
  year={2013},
  volume={},
  number={},
  pages={1-9},
  doi={10.1109/ICCCN.2013.6614155}}

@INPROCEEDINGS{8424662,
  author={M. Ghourab, Esraa and Samir, Effat and Azab, Mohamed and Eltoweissy, Mohamed},
  booktitle={2018 IEEE Security and Privacy Workshops (SPW)}, 
  title={Diversity-Based Moving-Target Defense for Secure Wireless Vehicular Communications}, 
  year={2018},
  volume={},
  number={},
  pages={287-292},
  doi={10.1109/SPW.2018.00046}}

@INPROCEEDINGS{6175633,
  author={Groat, Stephen and Dunlop, Matthew and Urbanksi, William and Marchany, Randy and Tront, Joseph},
  booktitle={2012 IEEE PES Innovative Smart Grid Technologies (ISGT)}, 
  title={Using an IPv6 moving target defense to protect the Smart Grid}, 
  year={2012},
  volume={},
  number={},
  pages={1-7},
  doi={10.1109/ISGT.2012.6175633}}

@article{WANG201410,
title = {A moving target DDoS defense mechanism},
journal = {Computer Communications},
volume = {46},
pages = {10-21},
year = {2014},
issn = {0140-3664},
doi = {https://doi.org/10.1016/j.comcom.2014.03.009},
author = {Huangxin Wang and Quan Jia and Dan Fleck and Walter Powell and Fei Li and Angelos Stavrou}
}

@ARTICLE{benvenuti2024guaranteed,

  author={Benvenuti, Alexander and Bialy, Brendan and Dennis, Miriam and Hale, Matthew},

  journal={IEEE Control Systems Letters}, 

  title={Guaranteed Feasibility in Differentially Private Linearly Constrained Convex Optimization}, 

  year={2024},

 volume={8},

  number={},

  pages={2745-2750}}

@book{boyd2004convex,
  title={Convex optimization},
  author={Boyd, Stephen and Vandenberghe, Lieven},
  publisher={Cambridge University Press},
  year={2004}
}

\end{document}